\documentclass[12pt]{article}
\usepackage{amsfonts}
\usepackage{wrapfig}
\usepackage{graphicx}
\usepackage{mathtools}
\usepackage{bbm}
\usepackage{hyperref}
\usepackage{bm}
\usepackage{booktabs}

\DeclareMathOperator*{\argmax}{arg\,max}
\DeclareMathOperator*{\argmin}{arg\,min}
\newcommand{\pr}{\mathbb{P}}
\newcommand{\ep}{\mathbb{E}}

\usepackage{amsthm}

\newtheorem{theorem}{Theorem}[section]

\newtheorem{lemma}[theorem]{Lemma}
\newtheorem{corollary}[theorem]{Corollary}
\newtheorem{definition}[theorem]{Definition}
\newtheorem{assumption}[theorem]{Assumption}
\newtheorem{remark}[theorem]{Remark}

\usepackage[margin=1in]{geometry}

\usepackage{subcaption}
\usepackage{verbatim}
\usepackage{fancyvrb}
\usepackage{xcolor}
\usepackage{algorithm}
\usepackage{algpseudocode}
\makeatletter
\newenvironment{breakablealgorithm}
  {
   \begin{center}
     \refstepcounter{algorithm}
     \hrule height.8pt depth0pt \kern2pt
     \renewcommand{\caption}[2][\relax]{
       {\raggedright\textbf{\fname@algorithm~\thealgorithm} ##2\par}%
       \ifx\relax##1\relax 
         \addcontentsline{loa}{algorithm}{\protect\numberline{\thealgorithm}##2}%
       \else 
         \addcontentsline{loa}{algorithm}{\protect\numberline{\thealgorithm}##1}%
       \fi
       \kern2pt\hrule\kern2pt
     }
  }{
     \kern2pt\hrule\relax
   \end{center}
  }
\makeatother

\usepackage[round]{natbib}
\usepackage{authblk}

\usepackage{setspace} 

\title{Policy-Oriented Binary Classification: Improving (KD-)CART Final Splits for Subpopulation Targeting}
\author[1]{Lei Bill Wang}
\author[2]{Zhenbang Jiao}
\author[2]{Fangyi Wang}
\affil[1]{Ohio State University Department of Economics}
\affil[2]{Ohio State University Department of Statistics}
\date{} 

\begin{document}
\maketitle

\begin{abstract}
  Policymakers often use recursive binary split rules to partition populations based on binary outcomes and target subpopulations whose probability of the binary event exceeds a threshold. We call such problems Latent Probability Classification (LPC). Practitioners typically employ Classification and Regression Trees (CART) for LPC. We prove that in the context of LPC, classic CART and the knowledge distillation method, whose student model is a CART (referred to as KD-CART), are suboptimal. We propose Maximizing Distance Final Split (MDFS), which generates split rules that strictly dominate CART/KD-CART under the unique intersect assumption. MDFS identifies the unique best split rule, is consistent, and targets more vulnerable subpopulations than CART/KD-CART. To relax the unique intersect assumption, we additionally propose Penalized Final Split (PFS) and weighted Empirical risk Final Split (wEFS).
  Through extensive simulation studies, we demonstrate that the proposed methods predominantly outperform CART/KD-CART. When applied to real-world datasets, MDFS generates policies that target more vulnerable subpopulations than the CART/KD-CART.

\end{abstract}

\section{Introduction}  \label{intro}
Policymakers and researchers often use recursive binary split rules to partition a population based on a binary outcome $Y\in\{0,1\}$ and target subpopulations with a probability of $Y = 1$ greater than 50\% when implementing a policy. 
We call such policy targeting problems Latent Probability Classification (LPC). Practically, CART is often employed for LPC. 
For example, \citet{andini2018targeting} uses CART to find subpopulations with a higher than 50\% probability of being financially constrained and recommends targeting these households with a tax credit program. 
In Appendix \ref{appendix Extensive use of CART for policy targeting}, we cite 23 empirical cases of using CART for various LPC problems to demonstrate the broad applicability of the setup that this paper investigates. 
In these studies, researchers typically use CART to divide the samples into many nodes, estimate the probability of $Y = 1$ for each node, and target those nodes with estimated probabilities higher than a threshold of 50\%, denoted as $c = 0.5$. This approach, though intuitive, is not optimal for LPC. 
Here, we provide a toy example in Figure \ref{Figure Sine example} to illustrate the limitation of using CART for an LPC problem.

\textbf{A toy example}: 
Suppose the latent probability of a binary event $Y = 1$ is a sinusoidal function of an observable variable $X$, $\mathbb{P}(Y=1|X) = \frac{\sin (2 \pi X)+1}{3} \text{ where $ X \sim {\rm Unif}[0,1]$}$. Figure \ref{Figure Sine example} shows the function, where the green segment represents the subpopulation that should be targeted, and the orange segments represent the subpopulation that should not be targeted.
In this example, we impose that the policymaker can only split the population \textit{once} based on the value of $X$.  
CART splits at $X = 0.5$, denoted as $s^{CART}$ in Figure \ref{Figure Sine example}. The left node has $\mathbb{P}(Y=1 | X > 0.5)  = \frac{2}{3}(0.5 + \frac{1}{\pi}) > 0.5$, whereas the right node has $\mathbb{P}(Y=1| X \leq 0.5)  = \frac{2}{3}(0.5 - \frac{1}{\pi}) \leq 0.5$. Consequently, we target the subpopulation by $X \leq 0.5$, i.e., left node. 
To demonstrate why $s^{CART}$ is suboptimal for the LPC problem, consider an alternative split at $s^* = \frac{5}{12}$. The left node still has $\mathbb{P}(Y=1|X \leq \frac{5}{12}) \approx 0.571 > 0.5$, right node has $\mathbb{P}(Y=1|X >\frac{5}{12}) \approx 0.164 < 0.5$. Only the left node ($X \leq
\frac{5}{12}$) is targeted. 
\begin{figure*}[ht]
    \centering
    \begin{subfigure}{0.3\textwidth}
        \centering
        \includegraphics[width=\linewidth]{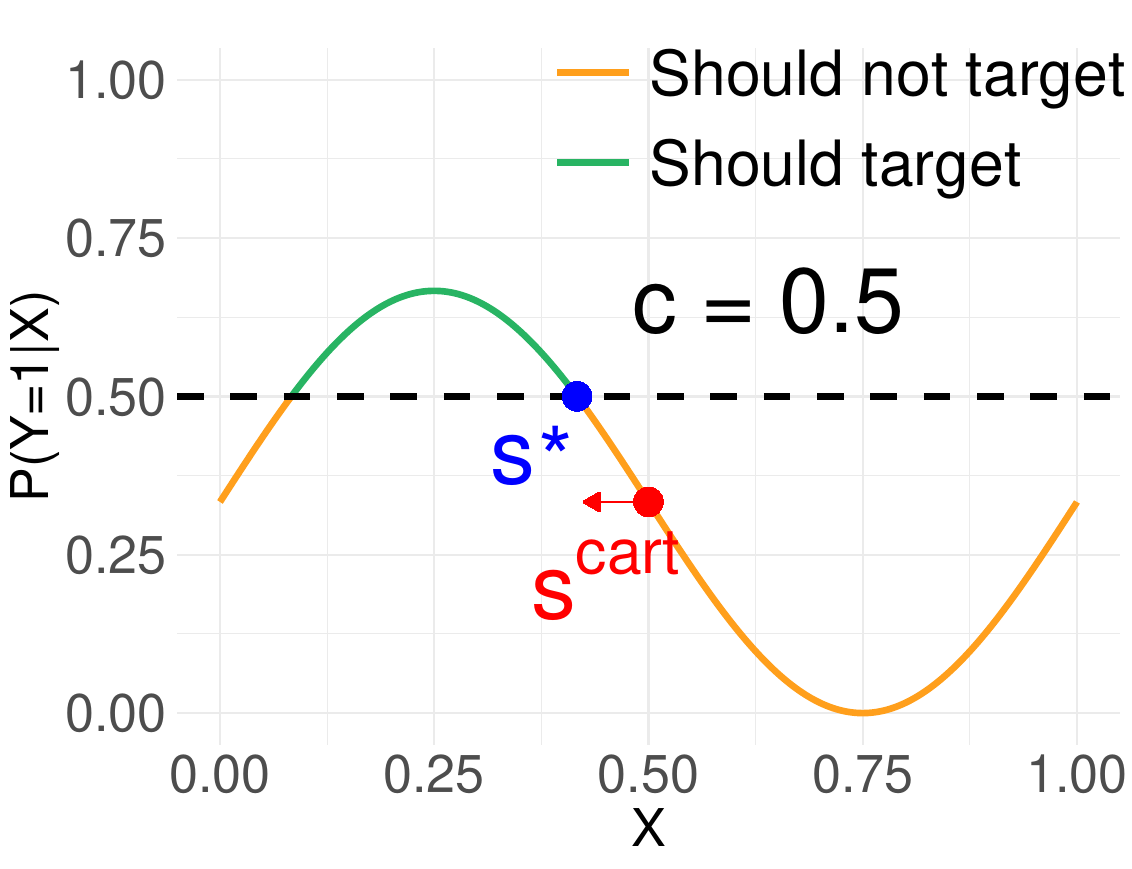}
        \caption{Toy example}
        \label{Figure Sine example}
    \end{subfigure}
    \hfill
    \begin{subfigure}{0.3\textwidth}
        \centering
        \includegraphics[width=\linewidth]{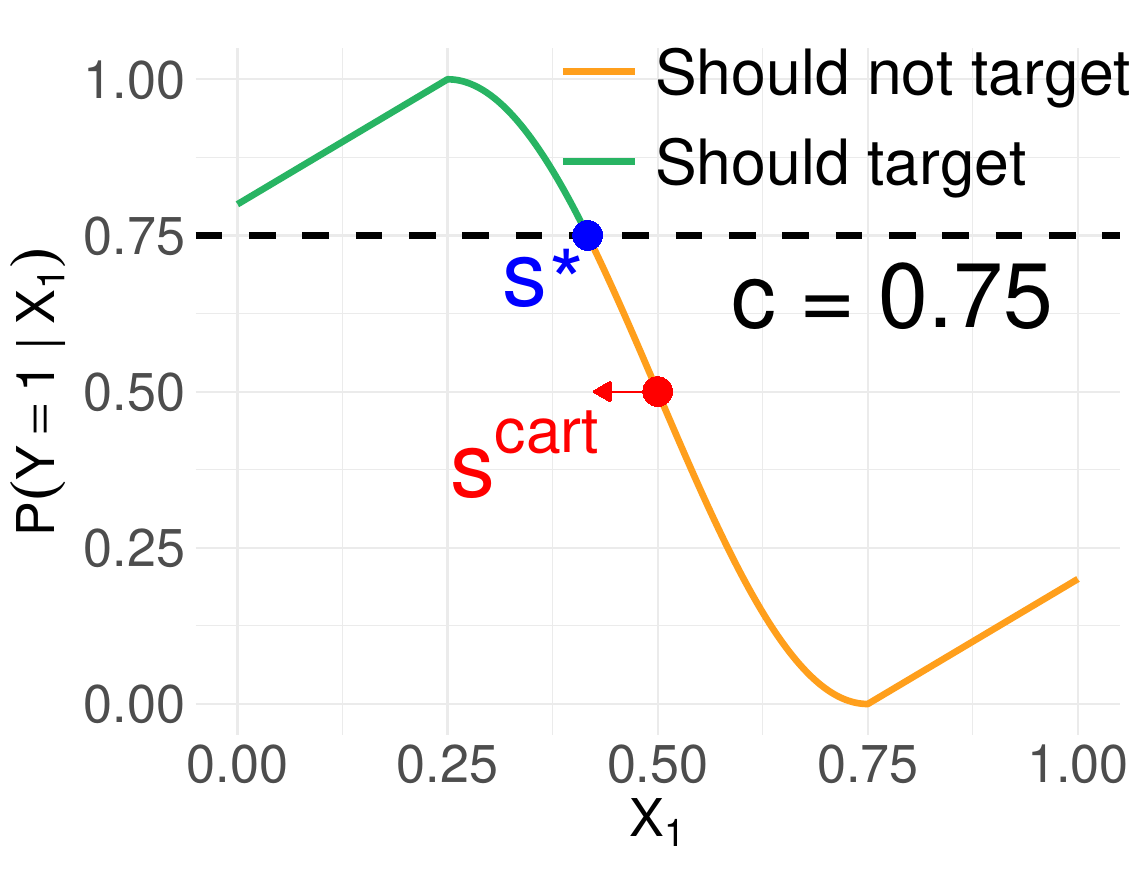}
        \caption{\(t_1\) example}
        \label{Figure Unique Sine example}
    \end{subfigure}
    \hfill
    \begin{subfigure}{0.3\textwidth}
        \centering
        \includegraphics[width=\linewidth]{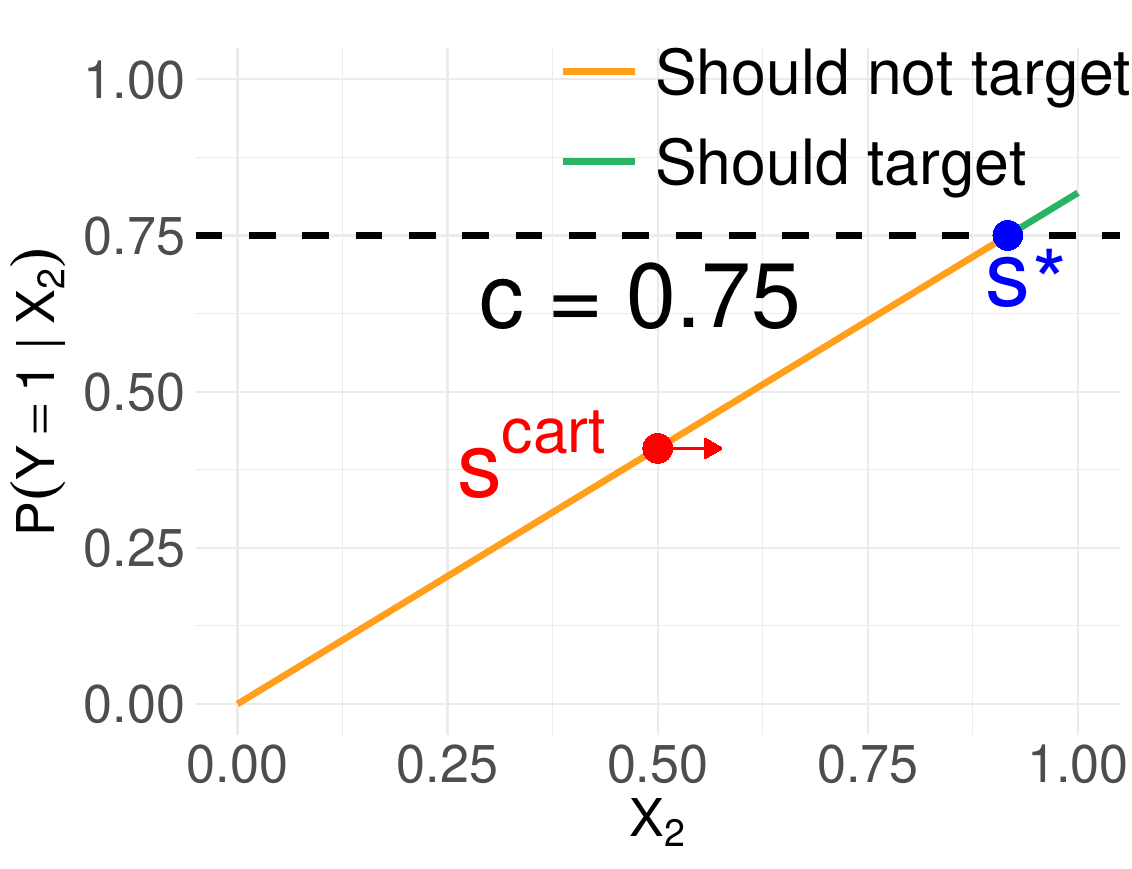}
        \caption{\(t_2\) example}
        \label{Figure Monotonic example}
    \end{subfigure}
    \caption{Illustrative examples comparing \(s^{CART}\) with \(s^*\)}
    \label{fig:policy_significance}
    \vspace{-15pt}
\end{figure*}
All subgroups that are correctly targeted/not targeted by $s^{CART}$ are also correctly targeted/not targeted by $s^*$. Moreover, $s^*$ excludes the group with $\frac{5}{12} < X < 0.5$ whose $\mathbb{P}(Y=1|X) < 0.5$ from being targeted whereas $s^{CART}$ incorrectly targets this subgroup. We say that $s^*$ strictly dominates $s^{CART}$. Section \ref{sec problem setup} formally defines strict domination.  

\textbf{Proposed methods}: 
This paper proposes methods that generate improved split rules. First, we replace the CART impurity function with a weighted sum of the distances between node means and the threshold $c$. This method is called Maximizing Distance Final Split (MDFS). Assuming the existence of a unique intersection between $\mathbb{P}(Y=1|X)$ and $c$, MDFS identifies the unique best split (hence, strictly dominating CART). The second method, Penalized Final Split (PFS), is a generalization of MDFS. It relaxes the unique intersect assumption and still strictly dominates CART. The third method, weighted Empirical risk Final Split (wEFS), adapts the weighted loss function from the cost-sensitive binary classification literature to our setup.


\textbf{Generalizing the threshold from $c = 0.5$ to $c \in (0,1)$}: 
Though most of the LPC studies use 50\% as the threshold ($c=0.5$ corresponds to the classic 0-1 loss), there are more general choices. For example, \citet{sarkar2024ensembling} uses CART to determine which forest zone has a higher than $c=61\%$ probability of forest fire to implement early warning systems. In some other cases, a policymaker may face budget constraints and hence, decides to adopt a threshold that is close to 1 \citep{hassanzadeh2021tradeoffs}. In the rest of this paper, we use $c \in (0,1)$ to denote the threshold. We assume that $c$ is \textit{fixed} before implementing our methods. In particular, policymakers can first tune $c$ with CART and then implement our methods with the chosen $c$. In this scenario, our methods still improve over CART because our theoretical results apply broadly to \textit{any} fixed $c \in (0,1)$.

\textbf{Extending to Knowledge distillation}: 
Knowledge distillation (KD) refers to a two-step learning algorithm. The first step trains a teacher model with a higher learning capacity, e.g., a neural network or a random forest, to learn $\mathbb{P}(Y=1|X)$. The second step uses the learned $\mathbb{P}(Y=1|X)$ as the response variable to train a simpler student model. 
The goal of the student model is to output a simple and interpretable representation of the teacher model's knowledge. In our case, the student model is a CART.
We refer the existing KD method with a CART student model as KD-CART. We apply the MDFS method to improve KD-CART and refer our proposed new method as KD-MDFS. This generalizes our contribution to \textit{a larger class of advanced tree-based methods}. 


\textbf{Summary of contributions:} Section \ref{sec problem setup} formulates the LPC problem from a wide range of empirical works and shows that split rule generated by CART/KD-CART is strictly dominated. Section \ref{sec penalized split} proposes MDFS, which point-identifies the unique best split rule assuming a unique intersection between $\mathbb{P}(Y=1|X)$ and $c$. In addition, we propose a consistent estimator for the MDFS split rule. Section \ref{sec policy significance} shows that MDFS generates policies that target \textit{more vulnerable} subpopulations. To relax the unique intersection assumption, we further propose PFS and wEFS in Section \ref{PFS and wEFS}. Lastly, in Section \ref{sec experiments}, we demonstrate that the proposed methods outperform their respective baselines and target more vulnerable subpopulations, using simulations with synthetic and real-world datasets.

\section{Related literature} \label{sec related literature}
The novelty of our problem setup can be best established by comparing our work with various strands of existing frontier literature.

\textbf{Nonparametric binary classification}: Nonparametric binary classifiers are often not designed for transparent policymaking. \cite{babii2024binary}'s theorems apply to uninterpretable deep learning. \cite{singh2022optimal} outputs stochastic decision rules, which raise fairness concerns. Our output policy is nonparametric, interpretable, and deterministic.

\textbf{Policy targeting}: Most of the existing policy targeting literature has been developed within the causal inference framework \cite{kitagawa2018should, athey2021policy, mbakop2021model}. Such causal inference methods are incompatible with the LPC setup, where the treatment might not have been tested in real life. A more detailed comparison between LPC and policy learning literature is provided in Appendix \ref{appendix comparing LPC and policy learning}.

\textbf{Tree-based methods' consistency}: The tree-based method literature has a vested interest in consistency. \cite{wager2018estimation,zheng2023consistency} show that different tree-based methods are consistent for estimating $\mathbb{P}(Y=1 | X)$. This work is interested in a different type of consistency. We show that our estimator consistently estimates the MDFS split rule.


\textbf{Knowledge distillation with CART as student}: \citet{liu2018improving,dao2021knowledge} use CART as the student model in knowledge distillation, termed as KD-CART in our paper. We show that the theoretical results we develop for MDFS apply to KD-CART. 

\section{Preliminaries: (OSF)LPC} \label{sec problem setup}

We first restrict the theoretical discussion to the following \textit{one-split, one-feature} (OSF) LPC problem characterized as follows. 
Let $\mathcal{X} = [0,1]$ and $\mathcal{Y} = \{0,1\}$ be the univariate feature space and label space, respectively. Let $f: \mathcal{X}\xrightarrow[]{}\mathbb{R}^{+}$ be a probability density function that is continuous on $\mathcal{X}$, $F$ be its corresponding cumulative distribution function, and $\eta(x) := \mathbb{P}(Y=1|X=x)$ be continuous.
Consider a dataset comprising $n$ i.i.d. samples, $(X_i, Y_i),\,i=1,\dots,n$, where $X_i\sim f$ and $Y_i|X_i\sim \text{Bernoulli}(\eta(X_i))$.
A policymaker is allowed to split the feature space one time (i.e. one-split) using the univariate feature $X$ (i.e. one-feature). After the split, policymakers target those node(s) whose node mean is greater than $c$. 

\subsection{CART is strictly dominated}
The most common criterion function optimized by CART to determine the split $s^{CART}$ is the weighted sum of variances of the two child nodes, i.e., $s^{CART} = \argmin_{s \in (0,1)} \mathcal{G}^{CART}(s)$
where $\mathcal{G}^{CART}(s) = F(s)(\mu_L(s) - \mu^2_L(s)) + (1-F(s))(\mu_R(s) - \mu^2_R(s))$, $\mu_L(s) =~ \int_0^s \eta(x) dF(x) / F(s)$ and $\mu_R(s) =~ \int_s^1 \eta(x) dF(x) / (1 - F(s))$ are the left and right node mean, respectively.

To illustrate that $s^{CART}$ is strictly dominated, 
we first introduce the following definition of strict dominance for comparing two splitting rules $s$ and $s'$. The definition is adapted from dominating decision rule in classic decision theory. To guide our readers through the dense notation in Definition \ref{Definition inadmissible}, we provide an intuitive explanation after introducing the definition.

\begin{definition}[Strict dominance] \label{Definition inadmissible}
    If $~\exists~ \{s,s'\} \in [0,1]^2$ such that $\forall\ x \in [0,1]$,$\mu_{t(x)}(s) > c \implies \mu_{t(x)}(s') > c  \quad \text{when $\eta(x) > c$}$, and $\mu_{t(x)}(s) \leq c \implies \mu_{t(x)}(s') \leq c \quad \text{when $\eta(x) \leq c$}$,
    where
    $\mu_{t(x)}(s) = \mu_L(s) ~\text{if} ~x \leq s$ and $\mu_{t(x)}(s) = \mu_R(s) ~\text{if} ~x > s$,
    \textbf{and} there exists a set $\mathcal{A}$ with a nonzero measure such that, $\forall\ x \in \mathcal{A}$, either one (or both) of the following conditions is true: (i) $\mu_{t(x)}(s) \leq c, \mu_{t(x)}(s') > c  \quad \text{when $\eta(x) > c$}$; and (ii) $\mu_{t(x)}(s) > c, \mu_{t(x)}(s') \leq c \quad \text{when $\eta(x) \leq c$}$.
    Then we say that splitting rule $s'$ \textbf{strictly dominates} splitting rule $s$.
\end{definition}
Despite its dense notation, Definition \ref{Definition inadmissible} has a straightforward interpretation: split rule $s'$ strictly dominates $s$ if it performs no worse  than $s$ for all $x \in [0,1]$ and strictly better than $s$ for all $x \in {\cal A} \subseteq [0,1]$, where $\cal A$ has nonzero measure.\footnote{``No worse'' means that if $s$ targets/not target correctly at a point, then $s'$ must also target/not target correctly at that point, ``strictly better'' means that at some points where $s$ target/not target incorrectly, $s'$ targets/not target correctly at those points.}

We show that under very general conditions, there exist some rules that strictly dominate CART's split rule in Theorem \ref{theorem tendency}. Proofs of all lemmas and theorems are collected in Appendix \ref{appendix proofs}.



\begin{theorem}\label{theorem tendency}
    Suppose $c\in [c_{min}, c_{max}]$ where $c_{min} = \min(\mu_L(s^{CART}), \mu_R(s^{CART})) $ and $ c_{max} = \max(\mu_L(s^{CART}) , \mu_R(s^{CART}))$ and $\eta(s^{CART}) \neq c$. Then there exists
    \begin{align*}
    s^* = \begin{cases}
        \begin{aligned}
            &\argmin_{s\in(0,s^{CART}), \eta(s) = c} (s^{CART} - s)\quad ~\text{ if }~ 
             (\eta(s^{CART}) - c)(\mu_R(s^{CART}) - \mu_L(s^{CART}))
             > 0,
        \end{aligned} \\
        \\
        \begin{aligned}
            &\argmin_{s\in (s^{CART},1), \eta(s) = c} (s - s^{CART})\quad ~\text{ if }~ 
            (\eta(s^{CART}) - c)(\mu_R(s^{CART}) - \mu_L(s^{CART})) < 0. 
        \end{aligned}
    \end{cases}
\end{align*}
    Further, all 
    $s \in \left((s^* \land s^{CART}), (s^* \lor s^{CART})\right)$, 
    \textbf{strictly dominates} $s^{CART}$.
\end{theorem}

The key challenge in understanding Theorem \ref{theorem tendency} is the interpretation of $s^*$. Figure \ref{Figure Sine example} provides a graphical illustration for interpreting $s^*$. Given that $\eta(s^{CART}) < c = 0.5$ and $\mu_L(s^{CART}) > \mu_R(s^{CART})$, Figure \ref{Figure Sine example} corresponds to the first minimization problem in the definition of $s^*$. The minimization problem searches for $s^*$ over the intersection of $s \in (0,s^{CART})$ and $s \in \{s: \eta(s) = c\}$, where $s^{CART} = 0.5$. In Figure \ref{Figure Sine example}, there are two candidate values, $s = \frac{1}{12}$ and $s = \frac{5}{12}$, between which $s=\frac{5}{12}$ is closer to $s^{CART}$, and hence $s^* = \frac{5}{12}$. The graphical illustration is generalizable: $s^{CART}$ determines which of the two minimization problems is used to determine $s^*$, then $\eta(s) = c$ pins down a set of candidate values of $s^*$, and lastly, $s^*$ is set to be the one that is closest to $s^{CART}$ among all candidate values.


\subsection{KD-CART is strictly dominated}
 In the LPC setup, a teacher model learns $\eta(x), x\in[0,1]$. Prediction based on the teacher model is denoted as $\hat{\eta}(x)$. In the LPC setup, the student model is a CART that takes in $\hat{\eta}(x)$ as the response and learns to partition the population based on $\hat{\eta}(x)$. One can show that Theorem \ref{theorem tendency} also applies to KD-CART, meaning there exist split rules that strictly dominates the split rule generated by KD-CART. Details are provided in Lemma \ref{lemma mid point for KD} in Appendix \ref{appendix proofs}.




\section{MDFS} \label{sec penalized split}

The suboptimality of CART for solving OSF LPC problem motivates our proposed method MDFS, which point-identifies $s^*$.
Since our theoretical results apply to OSF LPC, we advocate using our methods at the final splits with features identified by CART.\footnote{Though we focus on explaining MDFS in this section, all the theorems in this section easily extend to KD-MDFS, i.e., replacing the final splits' split criterion function of KD-CART with MDFS.}  
Restricting modifications to the last splits may appear trivial at first, however, note that as the tree grows deeper, the number of final splits increases exponentially, leading to non-trivial modifications to the policy designs. We substantiate this claim with empirical applications in Section \ref{sec empirical studies}. 



\subsection{Identification and estimation of MDFS
} \label{methods::mdfs}


\begin{assumption}[Unique intersection between $\eta(X)$ and $c$] \label{assumption unique intersection between eta and c}
     For $X \sim {\rm Unif}[0,1]$, there exists a unique $s^*$ such that $\eta(s^*) = c$, and $\eta(X)$ is strictly monotonic and differentiable on $[s^*-\epsilon,s^*+\epsilon]$ for some $\epsilon \in (0,\min(s^*,1-s^*))$.
\end{assumption}
We argue for the plausibility of Assumption \ref{assumption unique intersection between eta and c} in Section \ref{sec Plausibility}. 


\begin{theorem} \label{theorem unique intersection between eta and c}
    Under Assumption \ref{assumption unique intersection between eta and c}, 
    $\argmax_s \mathcal{G}^*(s,c)$ identifies $s^*$, where $\mathcal{G}^*(s,c) = s \left\lvert \mu_L - c \right\rvert + (1-s) \left\lvert \mu_R - c \right\rvert$.
    
    
\end{theorem}
The proof of Theorem \ref{theorem unique intersection between eta and c} can be largely decomposed into two steps: first, we show that there exists an interval containing $s^*$ in which first-order condition guarantees that $s^*$ is the local minimum; second, we show that for all $s$ outside of the interval, $\mathcal{G}^*(s,c) < \mathcal{G}^*(s^*,c)$.

Theorem \ref{theorem unique intersection between eta and c} states that $s^*$ maximizes the \textit{population} objective function ${\cal G}^*(s,c)$. In the \textit{finite} sample regime (denote sample size as $n$), we can estimate the sample version of $s^*$ using $\{(X_i,Y_i)\}_{i=1}^{n}$ by maximizing the sample analogue of ${\cal G}^*(s,c)$. Following from Assumption \ref{assumption unique intersection between eta and c}, we have $s^{*} \in (\epsilon,1-\epsilon)$. Define the MDFS final split estimator $\hat{s}$ as
\begin{align}
    \hat{s} = \argmax_{s\in(\epsilon,1-\epsilon)} \widehat{{\cal G}}^*(s,c),
    \notag
\end{align}
where
$\widehat{{\cal G}}^*(s,c)
    =
    s ~
    \Bigg|
        \frac
        {\sum_{i=1}^{n} Y_i\mathbbm{1}\{X_i \leq s\}}
        {\sum_{i=1}^{n}\mathbbm{1}\{X_i \leq s\}}
        -c
    \Bigg|+
    (1-s)
    \Bigg|
        \frac
        {\sum_{i=1}^{n} Y_i \mathbbm{1}\{X_i > s\}}
        {\sum_{i=1}^{n}\mathbbm{1}\{X_i > s\}}
        -c
    \Bigg|$
is the estimator of ${\cal G}^*(s,c)$. The following theorem states that $\hat{s}$ is a consistent nonparametric estimator for $s^*$.
\begin{theorem} \label{theorem MDFS consistency}
    Under Assumption \ref{assumption unique intersection between eta and c}, $\hat{s} \overset{p}{\to} s^*$
\end{theorem}

\subsection{Plausibility of Assumption \ref{assumption unique intersection between eta and c}} \label{sec Plausibility}
\textbf{Uniform $X$}. The set of quantile statistics for any continuous feature follows a uniform distribution, so we can convert a continuous variable to its quantile statistics. For a discrete $X$, the LPC problem is simpler. We defer the explanation to Appendix \ref{appendix proofs}.

\textbf{The unique intersection condition} is weaker than monotonicity, which is assumed by some theoretical works on CART \citep{blanc2020provable}. Moreover, in many real-life applications, the monotonicity of $\eta(X)$ is a reasonable assumption. For example, in our real-world dataset application in Section \ref{sec empirical studies}, a final node is split based on blood glucose level, and the binary outcome $Y$ is diabetic status. It is reasonable to assume that the probability of diabetes increases with blood glucose level. Nevertheless, monotonicity is \textit{not necessary}, see Figure \ref{Figure Unique Sine example} as a non-monotonic example that satisfies Assumption \ref{assumption unique intersection between eta and c}. 

Further, when we apply MDFS to the final nodes, the splitting feature is determined by CART. CART selects the splitting feature that gives the greatest reduction in variances (i.e., the greatest increase in purity in binary classification). Hence, when many features are available, CART is likely to pick some features that exhibit a salient trend (e.g., monotonicity), ruling out features that are more likely to violate the unique intersect assumption. Also, as we go down the tree, the domain of all nodes becomes smaller and smaller, making Assumption \ref{assumption unique intersection between eta and c} easier to satisfy. For example, take Figure \ref{Figure Sine example} as an example. It violates the unique intersect assumption. However, if we split the population at $X = 0.3$, then both child nodes will satisfy the unique intersection assumption.

\subsection{Policy significance of MDFS} \label{sec policy significance}
This section relaxes the OSF LPC setup and consider the policy relevance of MDFS under many-node many-feature setup.
We demonstrate two advantages of MDFS policies: (i) targeting more vulnerable subpopulations than CART by using the same amount of resources and (ii) uncovering subgroups with a higher-than-threshold probability of event $Y=1$ that CART ignores. 

Suppose there are two nodes $\{t_1,t_2\}$ with the same amount of population and corresponding features $X_1,X_2 \sim \rm{Unif}(0,1)$. We depict $\eta_1(X_1)$ for node $t_1$ in Figure \ref{Figure Unique Sine example} and $\eta_2(X_2)$ for node $t_2$ in Figure \ref{Figure Monotonic example}. The analytical forms of $\eta_1(X_1)$ and $\eta_2(X_2)$ are provided in Appendix \ref{appendix proofs}. Here we compare the targeted population using MDFS policy versus CART policy. Splitting nodes $t_1$ and $t_2$ individually at $s^{CART}$ versus $s^*$ results in different target subpopulations: $s^{CART}$: Target $\{X_1<\frac{1}{2}\}$ in $t_1$. $s^*$: Target $\{X_1<\frac{5}{12}\}$ in $t_1$ and $\{\frac{11}{12}<X_2<1\}$ in $t_2$.
The two sets of policies target the same proportion of the population, but $\eta_1(x_1) < 0.75$ for $x_1 \in \{\frac{5}{12}<X_1<\frac{1}{2}\}$, which is targeted by $s^{CART}$, whereas $\eta_2(x_2) > 0.75$ for $x_2 \in \{\frac{11}{12}<X_2<1\}$, which is targeted by $s^*$. Therefore, policies based on LPC target a \textbf{more vulnerable} subpopulation than CART/KD-CART policy. 

Admittedly, the fact that LPC and CART policies target the same proportion of the population in the previous example is by construction. It is also possible that the proportion of the population targeted by the MDFS policy is bigger than that by CART or KD-CART. In this scenario, comparing the effectiveness of the two sets of policies is not straightforward. 
Nonetheless, LPC still has policy significance. It discovers new latent groups with a higher-than-$c$ probability of an adversarial event happening to them (i.e., vulnerable subpopulations), e.g., $\eta_2(x_2) > 0.75$ for $\frac{11}{12} < x_2 < 1$. 

We make two remarks to formalize the two advantages that the MDFS policy offers. Assuming a homogeneous targeting cost per unit of population, the targeting cost of a policy is the percentage of the subpopulation targeted. 
Denote all $M$ final splitting nodes as $\{t_1,t_2,\dots,t_M\}$ and $\mathcal{M} = \{1,2,\dots,M\}$. For node $m$, we denote the feature selected by CART as $X_{(m)}$ and the CDF of $X_{(m)}$ as $F_m$ and let $\eta_m(x) = \pr(Y=1|X_{(m)} = x)$. The cost of CART and MDFS policies are 
\begin{align*}
    C^{CART}=& \sum_{m \in \mathcal{M}} \int_{0}^{1} \mathbbm{1}\{\mu_{t(x)}(s^{CART}_m) > c\} dF_m(x) \\
    C^{*}=& \sum_{m \in \mathcal{M}} \int_{0}^{1} \mathbbm{1}\{\mu_{t(x)}(s^{*}_m) > c\} dF_m(x)
\end{align*}
Assume that for some 
$~ m \in \mathcal{M}, 
~s^{CART}_m \neq s^{*}_m$.
\begin{remark} \label{remark same cost}
    If $C^{CART} = C^{*}$, then MDFS policy targets strictly more vulnerable $(\text{greater}~ \eta_m(X_{(m)}))$ subpopulation than CART using the same targeting resources.
\end{remark}
\begin{remark} \label{remark more cost}
    If $C^{CART} < C^{*}$, then MDFS policy uncovers latent subgroups in some node $m$ with selected feature $X_{(m)} = x$ whose $\eta_m(x) > c$ that CART ignores.
\end{remark}

\section{PFS and wEFS}
\label{PFS and wEFS}
To relax Assumption \ref{assumption unique intersection between eta and c}, we propose two additional methods: PFS and wEFS. 

\subsection{PFS
}

Intuitively, if $\mu_L$ is \textit{close} to $c$ and it is slightly higher than $c$, by the continuity assumption on $\eta(x)$, it's likely that 
there is a considerable amount of subpopulation from the left node with $\eta(x) < c$, see $0<x<\frac{1}{12}$ and $\frac{5}{12}<x<\frac{1}{2}$ in Figure \ref{Figure Sine example} as examples. This can substantially increase misclassification cases because $\mu_L$ and $\eta(x)$ are on different sides of $c$ for all these $x$ values.
Following this intuition, we consider adding a penalty to $\mathcal{G}^{CART}$ that pushes $\mu_L$ and $\mu_R$ away from $c$. Let
\begin{align}
    &\mathcal{G}^{PFS}(s,c) 
    = \mathcal{G}^{CART} + \lambda J \label{G_pfs}
\end{align}
where $J = F(s) W(|\mu_L - c|) + (1-F(s)) W(|\mu_R - c|)$ and $W: \mathbb{R}^{+} \cup 0 \to \mathbb{R}^{+}\cup 0$ is a decreasing function that penalizes small distances between the node means (i.e., $\mu_L$ and $\mu_R$) and $c$, and $\lambda\ge0$ controls the weight of the penalty term, which we denoted as $J$. The following theorem states that under some regularity conditions of the function $W$ and weight $\lambda$, the split rule that minimizes ${\cal G}^{PFS}$, denoted as $s^{PFS}$, strictly dominates $s^{CART}$.  



\begin{theorem} \label{theorem penalized loss works}
    Suppose $W: \mathbb{R}^+ \xrightarrow[]{} \mathbb{R}^+$ is convex, monotone decreasing, and upper bounded in second derivative. If $s^{CART}$ is the unique minimizer for $\mathcal{G}^{CART}$ and $\eta(s^{CART}) \neq c$, then there exists a $\Lambda > 0$, such that $\forall \lambda \in (0,\Lambda)$, $s^{PFS} := \argmin_{s\in [0,1]} \mathcal{G}^{PFS}(s,c)$ strictly dominates $s^{CART}$.
\end{theorem}

We provide an intuitive explanation for Theorem \ref{theorem penalized loss works} in Appendix \ref{appendix proofs}. Note that MDFS is a special case of PFS whose $\lambda = \infty$ and $W$ is the identity function.

\subsection{wEFS}
We adapt cost-sensitive classifier from \citet{koyejo2014consistent} to our setup and name it weighted Empirical Final Split (wEFS). Due to the lack of theoretical backing, we do not elaborate on wEFS. We defer the explanation for the lack of the theoretical results to Section \ref{sec conlusion}. Details about wEFS can be found in Algorithm~\ref{Calculate_risk} in Appendix~\ref{algor}. 

\section{Experiments} \label{sec experiments}
In this section, we conduct comprehensive numerical experiments to compare the performance of our proposed methods and the classic CART under both regular and KD frameworks. 
Under regular framework, we compare the classic CART with MDFS, PFS, and wEFS. Under KD framework, we compare RF-CART with RF-MDFS. We replace KD with RF because the teacher model is a \textbf{R}andom \textbf{F}orest.
 

\subsection{Synthetic data} \label{sec simulation}
\textbf{Settings:}
We simulate synthetic datasets using 8 different data generation processes (DGP), e.g., the Friedman synthetic datasets \citep{friedman1991multivariate, breiman1996bagging}. These setups are designed to capture various aspects of real-world data complexities and challenges commonly encountered when using tree-based methods, such as nonlinear relationships, feature interactions, collinearity, and noises (subtitles in Figure~\ref{Boxplot simulation results in main text}). Detailed descriptions of the DGP are provided in Appendix~\ref{Data generation details}.
For each DGP, we consider a threshold of interest $c\in\{0.5, 0.6, 0.7, 0.8\}$, resulting in 32 unique tasks.

We set two stopping rules for growing the tree: (i) a max depth of $m\in\{4, 5, 6, 7\}$, (ii) a minimal leaf node size of $\rho n$, where $\rho \in \{1\%, 2\%, 3\%\}$ and $n$ is the sample size. We set $n = 5000$. The fitting procedures stop once either of the rules is met. These give us 12 configurations for each tasks. For each of the $32\times 12 = 384$ settings, we do 50 replicates of experiments.

\textbf{Evaluation metrics}: 
We define false positives (FP) and false negatives (FN) before defining the two performance metrics that we use in the synthetic data simulations.
\begin{align*}
    FP =& \sum_{i=1}^n\mathbbm{1}\{\hat{T}(\mathbf{X}_i) \leq c\}\mathbbm{1}\{\eta(\mathbf{X}_i) > c\} \\
    FN =& \sum_{i=1}^n\mathbbm{1}\{\hat{T}(\mathbf{X}_i) > c\}\mathbbm{1}\{\eta(\mathbf{X}_i) \leq c\}
\end{align*}
where $\mathbf{X}_i$ is the \textit{multivariate} feature vector which includes \textit{all} features in the simulation and $\widehat{T}(\mathbf{X}_i)$ is the tree-estimate of $\eta(\mathbf{X}_i)$.

We use two metrics to evaluate the performance and robustness of all methods: misclassification rate (MR) and F1 score (F1). MR and F1 are defined as
\begin{align*}
    \text{MR} =& \frac{1}{n} (FP + FN) \\
    \text{F1} =& \frac{1}{n} \frac{2(n-FP)}{2(n-FP)+FP+FN}
\end{align*}
Given that our theorems state that split rules generated by (RF-)MDFS and PFS strictly dominate (RF-)CART, we expect that our proposed methods outperform their counterpart in both metrics.



\textbf{Model fitting procedures}: As we mentioned in Section~\ref{sec penalized split}, our methods and their counterparts 
only differ in the splits at the final level (final splits).
We first implement CART until the final split is reached based on the stopping rules. 
Once the final split is identified, we select the splitting feature using CART, then implement MDFS, PFS and wEFS with the selected splitting feature. MDFS and PFS decide the split by optimizing $\mathcal{G}^*$ and $\mathcal{G}^{PFS}$, respectively, while wEFS selects the split by identifying $s$ such that the minimum weighted empirical risk is obtained.

For $W$ in $\mathcal{G}^{PFS}$ from \eqref{G_pfs}, we choose $W(d) = 1 - d$. This choice of $W$ satisfies conditions specified in Theorem \ref{theorem penalized loss works}. 
We also experiment with alternative choices, such as $W(d)=(1-d)^2$ and $W(d)=\exp(-d)$, but observed similar performance across these choices.
Based on Theorem~\ref{theorem penalized loss works}, we set $\lambda=0.1$, a sufficiently small value that presumably satisfies the conditions of the theorem while maintaining practical effectiveness. We also propose a standard $\lambda$ selection procedure inspired by cross-validation and the honest approach \citep{athey2016recursive} (see details in Appendix~\ref{Appendix: honest approach}). 

\textbf{Comparison against respective baselines}: 
Figure~\ref{Boxplot simulation results in main text} provides a pairwise evaluation of MR differences between CART and its three refinements—MDFS, PFS and wEFS, and between RF-CART and RF-MDFS.
Each boxplot depicts the distribution of the MR difference for a single DGP across 50 replicates, with $m=7, \rho = 2\%, c=0.5$. Tables beneath the boxplots report one-sided p-values from paired t-tests (mean $<$ 0) and Wilcoxon signed-rank tests (median $<$ 0).  Except for a few cases in Ball, Friedman \#1, and Ring, MDFS and RF-MDFS reduce both the mean and the median MR relative to their counterparts, CART and RF-CART, at the 5$\%$ significance level. PFS and wEFS outperform CART in most cases, but have less stability than MDFS, as evidenced by their large p-values in a few cases.  

The comparison in terms of F1 is similar. All proposed methods largely outperform their counterparts at the 5$\%$ significance level. We defer this result to Appendix \ref{Appendix: simulation result}.

\begin{figure*}[htb]
    \centering
    \includegraphics[width=\linewidth]{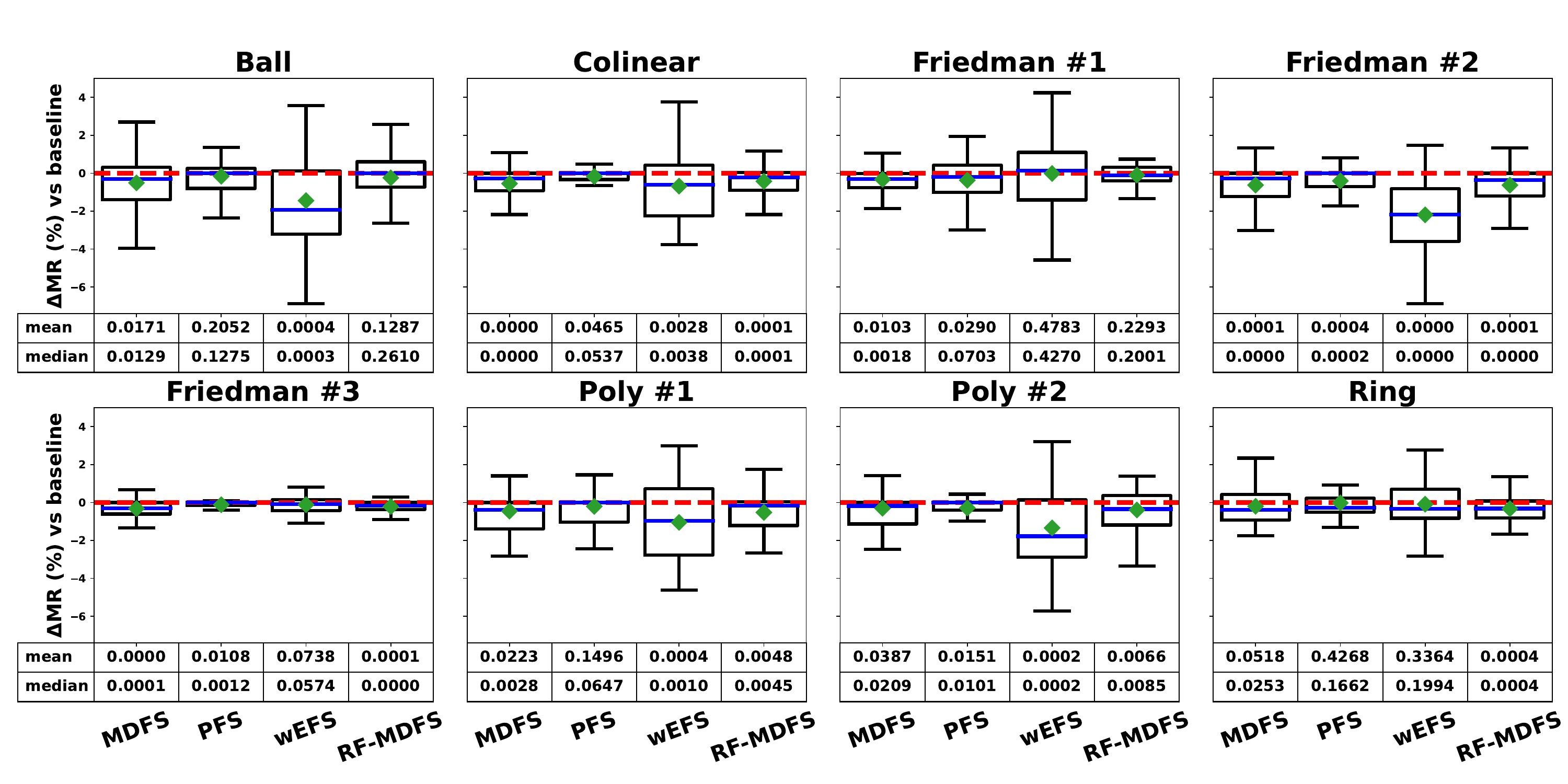}
    \caption{
    Boxplots of MR differences relative to baseline models. For each panel, the first three boxplots compare CART with MDFS, PFS and wEFS, and the last boxplot compares RF-CART with RF-MDFS. 
    Negative values in the boxplots denote \textit{improvement} in MR. Embedded tables list one-sided paired t-test and Wilcoxon signed-rank p-values for mean and median MR differences, respectively. See the boxplot for F1 scores in Appendix \ref{Appendix: simulation result}.}
    
    
    \label{Boxplot simulation results in main text}
\end{figure*}


\textbf{RF-MDFS has the strongest performance among six methods}: Our methods have a clear advantage over their respective baselines, as shown by Figure \ref{Boxplot simulation results in main text}. In addition, we compare all six methods together, instead of pair-wise comparison. RF-MDFS performs the best out of all six methods on 26 out of the 32 tasks in terms of MR, 27 out of 32 tasks in terms of F1, as shown by Table \ref{Table simulation results in appendix} and Table \ref{Table simulation results in appendix (F1)} in Appendix~\ref{Appendix: simulation result}. These results highlight the strength of combining MDFS with KD-CART, a frontier tree-based algorithm.

\subsection{Real-world datasets} \label{sec empirical studies}
We implement CART, MDFS, RF-CART and RF-MDFS with the
Pima Indians Diabetes dataset \citep{smith1988using} 
to demonstrate the policy significance of our proposed methods.\footnote{We supplement an additional forest fire empirical study in Appendix \ref{Appendix Empirical Studies} to showcase the wide applicability of our paper.} 
The Pima Indians Diabetes dataset measures health factors among Pima Indian women with the response variable being a binary indicator of diabetes status: 34.9\% of the sample is diabetic. It contains 768 observations and 8 features: number of pregnancies, glucose level, blood pressure, skin thickness, insulin level, BMI, family diabetes history index, and age. We use these eight features to search for subpopulations whose probability of having diabetes is above 60\% with depth of trees fixed at $m=3$. We pick these hyperparameter values so that one set of the empirical results match Remark \ref{remark same cost}. We also experimented with other hyperparameter values. The results from these additional experiments align with Remark \ref{remark more cost}. 
\begin{figure*}[ht] 
    \centering
    \begin{minipage}{0.5\textwidth}
        \centering
        \fontsize{8pt}{9.6pt}\selectfont
        \textbf{CART}
            \begin{Verbatim}[commandchars=\\\{\}]
if Glucose > 127.5
    if BMI <= 29.95
        if Glucose <= 145.5
            value: 0.146, samples: 41
            value: 0.514, samples: 35
        if Glucose <= 157.5
            \textcolor[HTML]{990000}{value: 0.609, samples: 115} 
            \textcolor[HTML]{990000}{value: 0.870, samples: 92}
        \end{Verbatim}
    \end{minipage}%
    \hfill
    \begin{minipage}{0.5\textwidth}
        \centering
        \fontsize{8pt}{9.6pt}\selectfont
        \textbf{MDFS}
        \begin{Verbatim}[commandchars=\\\{\}]
if Glucose > 127.5
    if BMI <= 29.95
        if Glucose <= 166.5
            value: 0.250, samples: 64
            \textcolor[HTML]{990000}{value: 0.667, samples: 12}
        if Glucose <= 129.5
            value: 0.579, samples: 19 
            \textcolor[HTML]{990000}{value: 0.739, samples: 188}
        \end{Verbatim}
    \end{minipage}
\\
\vspace{.5cm}
\begin{minipage}{0.5\textwidth}
        \centering
        \fontsize{8pt}{9.6pt}\selectfont
        \textbf{RF-CART}
            \begin{Verbatim}[commandchars=\\\{\}]
if Glucose > 127.5
    if BMI <= 29.95
        if Glucose <= 145.5
            value: 0.193, samples: 41
            value: 0.511, samples: 35
        if Glucose <= 157.5
            value: 0.595, samples: 115 
            \textcolor[HTML]{990000}{value: 0.830, samples: 92}
        \end{Verbatim}
    \end{minipage}%
    \hfill
    \begin{minipage}{0.5\textwidth}
        \centering
        \fontsize{8pt}{9.6pt}\selectfont
        \textbf{RF-MDFS}
        \begin{Verbatim}[commandchars=\\\{\}]
if Glucose > 127.5
    if BMI <= 29.95
        if Glucose <= 166.5
            value: 0.284, samples: 64
            \textcolor[HTML]{990000}{value: 0.636, samples: 12}
        if Glucose <= 129.5
            value: 0.573, samples: 19 
            \textcolor[HTML]{990000}{value: 0.713, samples: 188}
        \end{Verbatim}
    \end{minipage}

    \caption{The targeting policies generated by CART, MDFS, RF-CART, RF-MDFS. The \textcolor[HTML]{990000}{red} groups are the targeted subpopulations predicted to a higher than 60\% probability of being diabetic. We present nodes that differ by targeting decisions due to page limit, see the full trees in Appendix \ref{Appendix Empirical Studies}.}
    \label{figure forest fire CART vs MDFS}
\end{figure*}

\textbf{CART vs MDFS}: As shown in Figure \ref{figure forest fire CART vs MDFS}, CART and MDFS commonly target those with ${\rm Glucose} > 129.5$ and ${\rm BMI} > 29.95$, which consists of 188 observations. The two methods' targeting policies differ in two ways: CART additionally targets $127.5 < {\rm Glucose} \leq 129.5$ and ${\rm BMI} > 29.95$. This subgroup consists of 19 observations. MDFS additionally targets ${\rm Glucose} > 166.5$ and ${\rm BMI} \leq 29.95$. This subgroup consists of 12 observations.
The difference between the sizes of these two subgroups is 7, which is small relative to 188, i.e., the size of the subgroup commonly targeted by both policies. Assuming that the sample is representative of the population of all Pima Indian women, then the two sets of policies will incur a similar amount of targeting resources. The additional group targeted by CART has a 57.9\% probability of having diabetes, whereas the additional group targeted by MDFS has a 66.7\% probability of having diabetes. MDFS targets a subpopulation that is more prone to diabetes. This corresponds to Remark \ref{remark same cost}.

\textbf{RF-CART vs RF-MDFS}: As shown in Figure \ref{figure forest fire CART vs MDFS}, both RF-CART and RF-MDFS commonly target ${\rm Glucose} > 157.5$ and ${\rm BMI} > 29.95$, a subgroup consisting of 92 observations. RF-MDFS additionally targets two subgroups: ${\rm Glucose} > 166.5$ and ${\rm BMI} \leq 29.95$ consisted of 12 observations and $129.5 < {\rm Glucose} \leq 157.5$ and ${\rm BMI} > 29.95$, which consisted of 96 observations.
RF-MDFS targets a much greater number of observations than RF-CART. If the sample is representative of the population, then RF-MDFS would use much more targeting resources than RF-CART. Nevertheless, RF-MDFS is useful in uncovering groups with higher than 60\% probability of having diabetes that RF-CART is unable to find. For example, the first group that RF-MDFS additionally targets has a 63.6\% probability of having diabetes. This aligns with Remark \ref{remark more cost}.

\section{Further literature and conclusion}  \label{sec conlusion}

\textbf{Cost-sensitive binary classification}: LPC can be related to the cost-sensitive binary classification problem, further highlighting the policy significance of the LPC problem. \citet{nan2012optimizing,menon2013statistical,koyejo2014consistent} show that the optimal classifier of a cost-sensitive binary classification problem is determined by whether the latent probability is above some performance-metric-dependent threshold. 
Whether the theoretical results in the cost-sensitive classification literature carry over to our setup is unclear. For example, \citet{nan2012optimizing} shows risk consistency with uniform convergence of the empirical risk. In our setup, $\{\mu_L(s),\mu_R(s)\}$ appear inside the indicator functions, and uniform convergence of empirical risk is not guaranteed even with a large sample. In another related work, \citet{koyejo2014consistent} proves risk-consistency of weighted ERM assuming that the policymakers search through all real-valued functions. By definition, LPC searches binary-split-type policies only. The performance of wEFS is strong in our simulations, so it would be interesting to look into the theoretical property of wEFS in future research.

\textbf{Subgroup discovery}: \cite{lavravc2004subgroup,herrera2011overview,atzmueller2015subgroup,helal2016subgroup} discuss a widely-used quality measure, called Weighted Relative Accuracy (WRAcc), in the subgroup discovery literature. Under the assumption that $ X \sim {\rm Unif}[0,1]$, 
\begin{align*}
    \text{WRAcc}_L
    = F(s) \left( \frac{\int_{0}^{s}\eta(x) dx}{s} - \int_{0}^{1}\eta(x) dx \right)
\end{align*}
where the subscript $L$ indicates it is the WRAcc measure associated with the left node. $\mathcal{G}^*$ is a generalization of $\text{WRAcc}_L + \text{WRAcc}_R$. Hence, Theorem \ref{theorem unique intersection between eta and c} applies to a combination of CART and WRAcc.
\begin{remark} \label{remark generalizing WRAcc}
    Under Assumption \ref{assumption unique intersection between eta and c}, a CART algorithm that uses $\text{WRAcc}_L + \text{WRAcc}_R$ as the objective function is a special case of $\mathcal{G}^*$ whose $c = \ddot{s}$ where $\eta(\ddot{s}) = \int_{0}^{1} \eta(x) dx$. 
\end{remark}
\begin{corollary}
    Under Assumption \ref{assumption unique intersection between eta and c}, a CART algorithm that uses $\text{WRAcc}_L + \text{WRAcc}_R$ identifies $\ddot{s}$.
\end{corollary}

\textbf{Conclusion}: Our paper points out that classic CART/KD-CART is suboptimal for LPC in each split. 
Based on different assumptions, we propose three alternative methods: MDFS, PFS and wEFS. MDFS and PFS generate policy rules that strictly dominate CART. Our proposed methods predominantly outperform their counterparts in our simulation and provide more policy insights in real-world applications.

\bibliography{reference}
\bibliographystyle{plainnat}




\appendix
\section{Additional literature reivew}
\subsection{Extensive use of CART for LPC setups}\label{appendix Extensive use of CART for policy targeting}
\textbf{Traffic safety:} \citet{da2017identification} identifies potential sites of serious accidents in Brazil using different decision tree algorithms. Policymakers may choose to take more safety precautions at dangerous traffic spots whose probability of having a fatal accident is above a threshold. Other research that uses decision trees to identify dangerous traffic situations includes \citet{clarke1998machine,kashani2011analysis,obereigner2021methods}.

\textbf{Fraud detection:} \citet{sahin2013cost, save2017novel,lakshmi2018machine}  find conditions under which credit card fraudulent usage is likely to happen. Credit card companies can inform card owners when the probability of fraudulent usage is above a pre-specified threshold.


\textbf{Mortgage lending:} 
\citet{feldman2005mortgage, cciugcsar2019comparison, madaan2021loan} predict different subpopulations' mortgage (and other types of loan) default rates using a decision tree. Mortgage loan lenders can use the result to determine whether to deny a loan request. Another stream of literature on mortgage lending using decision trees focuses on racial discrimination \citep{varian2014big, lee2021algorithmic, zou2023ai}.  Policymakers may want to intervene in situations where with high enough probability race seems to play a role in determining whether mortgage lending is denied and the probability of denial is high.

\textbf{Health intervention:} \citet{mann2008classification, shouman2011using, crowe2017weight, speiser2018predicting, toth2021decision, mahendran2022quantitative} classify diabetes (or other diseases) using relevant risk factors. Such classification result leads to different health interventions. Doctors may recommend various treatments based on whether the subpopulation a patient belongs to is more likely to be classified as type I or type II diabetes.

\textbf{Water management:} One of the objectives of \citet{herman2018policy} is to manage flood control and output a threshold-based water resources management policy. The authors output a set of conditions that define the subpopulation whose probability of flooding is high enough to justify policy intervention. Many other studies also use decision trees to classify whether the water quality is satisfactory, resulting in important implications to water management policies \citep{waheed2006measuring, saghebian2014ground,hannan2021classification}.

\subsection{Comparing LPC and policy learning/causal subgroup/individualized treatment rules} \label{appendix comparing LPC and policy learning}
Policy targeting can be largely divided into two cases: policy targeting with treatment testing and policy targeting without treatment testing. The former is a burgeoning literature termed policy learning/causal subgroup/individualized treatment rules. The LPC framework falls under the latter, imposing no requirement on treatment testing. Here we provide two common scenarios where LPC is useful but policy learning is not possible.

\textit{What if the treatment cannot be/has not yet been tested in real life?}

Our methods do not need a randomized control trial/quasi-experiment/instrumental variable for unconfoundedness, which is a key assumption for the policy learning literature. Experimental data (or observational data that satisfies unconfoundedness) is not always available.

For instance, in the case of tax credit programs (the example in our introduction), it is nearly impossible to experiment since the tax credit program is a one-time, urgent assistance package for households during financial crisis. Policymakers are unlikely to test out the assistance package with a random set of households and then evaluate how to allocate these assistance packages at the population level. There are two reasons. One, such an experiment is ethically questionable; two, by the time such an experiment is completed, the financial crisis might have already been over.

In such cases, classic policy learning is impractical. LPC can help refine policy design by identifying and targeting the subgroups that are more likely to be financially constrained. It is much faster for the policymakers (in this case, the Italian central bank) to collect data on households' financial information (e.g., whether the household is financially constrained) than to experiment.

\textit{What if the policymakers do not have a specific treatment in mind?}

Another scenario that happens often is that the policymakers may decide the treatment after finding the vulnerable subgroups. In this case, testing the treatment is not possible, since there is no treatment to begin with.

Example 1: traffic safety

\cite{da2017identification} uses CART to find out that for the northern part of BR-116 (a highway in Sao Paulo, Brazil), the severe traffic rate (FSI in the paper) is higher than 50\% (severe means involving human injury/death, the rate is the number of severe events divided by the total number of accidents).
In contrast, the southern part of BR-116 has a higher than 50\% FSI during the peak hours (12 pm to 6 pm) when there is drizzle.
Policymakers can analyze these two geographic regions separately and use different interventions for the northern and southern parts of BR-116.

For the northern part, there might be some geographic features (maybe curvy turns) that make the road dangerous at all times, then installing traffic mirrors/stop signs/traffic lights at curvy turns would be good solutions.
For the southern part, since it is during specific times, then sending out more traffic police during those hours to check on the speed limit would work better with minimal disruption to ongoing traffic.

Example 2: welfare take-up

Another example is \cite{wang2025disentangling}, in their paper, they find subgroups among eligible households for WIC (a welfare program) who have a low probability of choosing to participate in WIC (they use random forest instead of CART). They then analyze which subgroup is less likely to participate for what reason: unawareness, limited usefulness, hassle or stigma. For different reasons, different interventions are suggested. For example, an information campaign is suggested for higher-educated subgroups, and choice-inducing strategies are suggested for those who are already in the programs.

In both examples, the treatment is decided after the subgroups are found. This is inherently contradictory to the premise of policy learning that treatment has to be tested before finding the targeting subgroups.

\section{Mathematical appendix} \label{appendix proofs}

We state a property of the CART splitting rule in Lemma~\ref{lemma mid point}. This is useful for establishing the rest of the theoretical results.


\begin{lemma} \label{lemma mid point}
$2\eta(s^{CART}) - \mu_L(s^{CART}) - \mu_R(s^{CART}) = 0$ and $\mu_L(s^{CART}) \neq \mu_R(s^{CART})$, 
\end{lemma}
Proof of Lemma~\ref{lemma mid point} needs the following lemma, which is proved after the proof of Lemma~\ref{lemma mid point}.
\begin{lemma}\label{lemma FOC}
    Given the problem setup, for any node $t$, assume that $\eta(X)$ is continuous and not a constant,  $\frac{\partial \mathcal{G}^{CART}(s)}{\partial s} = 0$ is a necessary condition for $s$ to be an optimal split point.
\end{lemma}

{\bf Proof for Lemma \ref{lemma mid point}.}
\begin{proof}
    For a fixed $s\in[0,1]$, the impurity score after the split is given as:
    \begin{align}
        \mathcal{G}^{CART}(s)
        &= (\mu_L - \mu_L^2)F(s) + (\mu_R - \mu_R^2)(1-F(s)) \label{impurity sum} \\
        &= E_t - \mu_L^2F(s) - \mu_R^2(1-F(s)),\nonumber
    \end{align}
    where $E_t = \int_0^1 \eta(x)f(x)dx$ is free of $s$. Moreover, 
    \begin{align}
    \frac{\partial \mu_L}{\partial s} =&\ \frac{\partial \int_0^s \eta(x)f(x)dx/F(s)}{\partial s} = \frac{f(s)}{F(s)}\left( \eta(s) -\mu_L \right) \label{dElds}\\
    \frac{\partial \mu_R}{\partial s} =&\ \frac{\partial \int_s^1 \eta(x)f(x)dx/(1-F(s))}{\partial s} = \frac{f(s)}{1-F(s)}\left(-\eta(s) + \mu_R\right)\label{dErds},
    \end{align}
    Using (\ref{dElds}) and (\ref{dErds}), we simplify the first-order derivative 
    \begin{align}
        \frac{\partial \mathcal{G}^{CART}(s)}{\partial s} =&\ -2\mu_LF(s)\frac{\partial \mu_L}{\partial s} -\mu_L^2f(s) -2\mu_R(1-F(s))\frac{\partial \mu_R}{\partial s} + \mu_R^2f(s) \nonumber\\
        =&\ f(s)\left( -2\mu_L\left(\eta(s) - \mu_L\right)-\mu_L^2 - 2\mu_R\left(-\eta(s) +\mu_R\right) + \mu_R^2 \right)\nonumber\\
        =&\ f(s) (-2\mu_L\eta(s) + E^2_{t_L} + 2\mu_R\eta(s) - E^2_{t_R})\nonumber\\
        =&\ f(s) (2\eta(s) - \mu_L - \mu_R)(\mu_R - \mu_L)\label{der_G}\\
        \propto&\ (2\eta(s) - \mu_L - \mu_R)(\mu_R - \mu_L)\nonumber. 
    \end{align}
By Lemma \ref{lemma FOC},  $\frac{\partial \mathcal{G}^{CART}(s)}{\partial s} = 0$ is a necessary condition for $s$ to be an optimal split point. 

Given that $f$ is strictly positive, we have that $\frac{\partial \mathcal{G}^{CART}(s)}{\partial s} = 0$ iff $(2\eta(s) - \mu_L - \mu_R)(\mu_R - \mu_L) = 0$. The next step is to rule out the possibility that $\mu_R - \mu_L = 0$ outputs a local maximum. 
\begin{align*}
    \frac{\partial^2 \mathcal{G}^{CART}(s)}{\partial s^2} =&\ \underbrace{(2\eta(s) - \mu_L - \mu_R)(\mu_R - \mu_L)}_{=0 \text{ by the first-order condition}}\frac{\partial f(s)}{\partial s} + \frac{\partial (2\eta(s) - \mu_L - \mu_R)(\mu_R - \mu_L)}{\partial s} f(s) \\
    =&\ \frac{\partial (2\eta(s) - \mu_L - \mu_R)(\mu_R - \mu_L)}{\partial s} f(s) \\
    \propto&\ \frac{\partial (2\eta(s) - \mu_L - \mu_R)(\mu_R - \mu_L)}{\partial s}
\end{align*}

Consider the second derivative of $\mathcal{G}(s)$ with respect to $s$ with $\mu_L = \mu_R$, by \eqref{dElds} and \eqref{dErds}:

\begin{align*}
    \frac{\partial^2 \mathcal{G}^{CART}(s)}{\partial s^2} \propto&\ \frac{\partial (2\eta(s) - \mu_L - \mu_R)}{\partial s} \underbrace{(\mu_R - \mu_L)}_{=0} + (2\eta(s) - \mu_L - \mu_R)\left( \frac{\partial \mu_R}{\partial s} - \frac{\partial \mu_L}{\partial s} \right) \\ 
    \propto&\ (2\eta(s) - \mu_L - \mu_R)\left( - \frac{\eta(s)}{1-F(s)} + \frac{\mu_R}{(1-F(s))} -\frac{\eta(s)}{F(s)} + \frac{\mu_L}{F(s)} \right)\\
    =&\ \frac{2\eta(s) - \mu_L - \mu_R}{F(s)(1-F(s))}\left( - \eta(s) + 
    \mu_L (1-F(s)) + \mu_R F(s)\right)
\end{align*}

If $\mu_L = \mu_R$, 
\begin{align*}
    \frac{\partial^2 \mathcal{G}^{CART}(s)}{\partial s^2} \propto&\ (2\eta(s) - \mu_L - \mu_R)(-\eta(s) + \mu_L) = -2(\eta(s) - \mu_L)^2.
\end{align*}
For $s$ s.t. $2\eta(s, t) - \mu_L - \mu_R \neq 0$ and $\mu_L = \mu_R$, $\frac{\partial \mathcal{G}^{CART}(s)}{\partial s} = 0$ and $\frac{\partial^2 \mathcal{G}^{CART}(s)}{\partial s^2} < 0$: $\mathcal{G}^{CART}(s)$ reaches its local maximum, which cannot be a global minimum. Such $s$ is not the optimal split. 

Global minimum must have $(2\eta(s) - \mu_L - \mu_R)(\mu_R - \mu_L) = 0$ and $\mu_R - \mu_L \neq 0$. Therefore, $2\eta(s, t) - \mu_L - \mu_R = 0$ holds for optimal split $s$ in any node $t$ and dimension $p$.
\end{proof}

{\bf Proof for Lemma \ref{lemma FOC}.}
\begin{proof}
    We first rule out the possibility that the optimal splitting happens at the boundary point. 

    Consider $s = 0$, then we write the impurity score as 
    \begin{align*}
        \mathcal{G}^{CART}(0) =&\ (\mu_L(0) - \mu_L^2(0))F(0) + (\mu_R(0) - \mu_R^2(0))(1-F(0)) \\
        =&\ \mu_R(0) - \mu_R^2(0) = \Bar{\eta} - \Bar{\eta}^2
    \end{align*}
    where $\Bar{\eta}$ denotes the mean probability of $Y=1$ for the entire node $t$.

    Since $\eta(X)$ is not a constant, we can find $s = s'$ such that $\mu_L(s') \neq \mu_R(s')$ and in general $\Bar{\eta} = \mu_L(s')F(s') + \mu_R(s') (1 - F(s'))$, where $0<F(s')<1$. The impurity score for $s = s'$ is 
    \begin{align*}
        \mathcal{G}^{CART}(s') = (\mu_L(s') - \mu_L^2(s'))F(s') + (\mu_R(s') - \mu_R^2(s'))(1 - F(s'))
    \end{align*}

    Using the equality $\Bar{\eta} = \mu_L(s')F(s') + \mu_R(s') (1 - F(s'))$, we can rewrite $\mathcal{G}(0)$ and show that it is strictly greater than $\mathcal{G}(s')$.

    \begin{align*}
        \mathcal{G}^{CART}(0) =&\ \mu_L(s')F(s') + \mu_R(s') (1 - F(s')) - (\mu_L(s')F(s') + \mu_R(s') (1 - F(s')))^2 \\
        =&\ \mu_L(s')F(s') + \mu_R(s') (1 - F(s')) - (\mu_L(s')F(s'))^2 - ((\mu_R(s') (1 - F(s')))^2\\
        &\ \quad - 2 (\mu_L(s')F(s')) (\mu_R(s') (1 - F(s')))\\
        =&\ \mu_L(s')F(s') + \mu_R(s') (1 - F(s')) - (\mu_L(s')F(s'))^2 - ((\mu_R(s') (1 - F(s')))^2 \\
        &\ \quad - \mu_L^2(s') F(s')(1 - F(s')) - \mu_R^2(s') F(s')(1 - F(s')) \\
        &\ \quad + \mu_L^2(s') F(s')(1 - F(s')) + \mu_R^2(s') F(s')(1 - F(s')) \\
        &\ \quad - 2 (\mu_L(s')F(s')) (\mu_R(s') (1 - F(s')))\\
        =&\ \mathcal{G}^{CART}(s') + \mu_L^2(s') F(s')(1 - F(s')) + \mu_R^2(s') F(s')(1 - F(s')) \\
        &\ \quad - 2 (\mu_L(s')F(s')) (\mu_R(s') (1 - F(s')))\\
        =&\ \mathcal{G}^{CART}(s') + (\mu_L(s') - \mu_R(s'))^2 F(s'))(1 - F(s')) > \mathcal{G}(s')
    \end{align*}
The inequality means that $s = 0$ can never be the optimal split. The proof for the case of $s = 1$ is similar. Hence, we show that the optimal split is not the boundary point.

Given that $\mathcal{G}^{CART}(s)$ is differentiable, its domain is closed and compact, and the boundary points are not the optimal splits. The first-order condition must be satisfied at all interior local optima (including the global minimum whose argument is the optimal splitting).
    
\end{proof}

\begin{lemma} \label{lemma dominate and small risk}
    If splitting rule $s'$ strictly dominates $s$, then $R(s') < R(s)$. 
\end{lemma}
{\bf Proof for Lemma~\ref{lemma dominate and small risk}}
\begin{proof}
Using Definition \ref{Definition inadmissible} and contrapositive,
\begin{itemize}
    \item when $\eta(x) > c$, $\mu_{t(x)}(s') \leq c \implies \mu_{t(x)}(s) \leq c$ and
    \item when $\eta(x) \leq c$, $\mu_{t(x)}(s') > c \implies \mu_{t(x)}(s) > c$.
\end{itemize}
Therefore, $\forall\ x \in [0,1]$, 
\begin{align*}
    \mathbbm{1}\{ \eta(x) > c \} \mathbbm{1}\{\mu_{t(x)}(s') \leq c\} 
    \leq 
    \mathbbm{1}\{ \eta(x) > c \} \mathbbm{1}\{\mu_{t(x)}(s) \leq c\}
    \\
    \mathbbm{1}\{ \eta(x) \leq c \} \mathbbm{1}\{\mu_{t(x)}(s') > c\} 
    \leq
    \mathbbm{1}\{ \eta(x) \leq c \} \mathbbm{1}\{\mu_{t(x)}(s) > c\}   
\end{align*}
Also, there exists a set $\mathcal{A} \subseteq [0,1]$ with nonzero measure such that, $\forall\ x \in \mathcal{A}$, either (or both) of the following conditions is true
\begin{align*}
    \mathbbm{1}\{ \eta(x) > c \} \mathbbm{1}\{\mu_{t(x)}(s') \leq c\} 
    <
    \mathbbm{1}\{ \eta(x) > c \} \mathbbm{1}\{\mu_{t(x)}(s) \leq c\}
    \\
    \mathbbm{1}\{ \eta(x) \leq c \} \mathbbm{1}\{\mu_{t(x)}(s') > c\} 
    <
    \mathbbm{1}\{ \eta(x) \leq c \} \mathbbm{1}\{\mu_{t(x)}(s) > c\}   
\end{align*}

    \begin{align*}
        R(s') =&\ \int_{\dot{X}_t}^{\ddot{X}_t} \left( \mathbbm{1}\{ \eta(x) > c \} \mathbbm{1}\{\mu_{t(x)}(s') \leq c\}
        +
        \mathbbm{1}\{ \eta(x) \leq c \} \mathbbm{1}\{\mu_{t(x)}(s') > c\}   \right) f(x) dx \\
        <&\ \int_{\dot{X}_t}^{\ddot{X}_t} \left( \mathbbm{1}\{ \eta(x) > c \} \mathbbm{1}\{\mu_{t(x)}(s) \leq c\} +
        \mathbbm{1}\{ \eta(x) \leq c \} \mathbbm{1}\{\mu_{t(x)}(s) > c\}\right)  f(x) dx = R_t(s)
    \end{align*}
\end{proof}

{\bf Proof for Theorem~\ref{theorem tendency}}
\begin{proof}
    By Lemma~\ref{lemma mid point}, without loss of generality\footnote{Proof for the case $\mu_R(s^{CART}) < \eta(s^{CART}) < \mu_L(s^{CART})$ is similar.}, assume 
    \begin{align*}
        \mu_L(s^{CART}) < \eta(s^{CART}) = \frac{\mu_L(s^{CART}) + \mu_R(s^{CART})}{2} < \mu_R(s^{CART}).
    \end{align*}
    We sequentially prove the two claims in the lemma: the existence of $s^*$ and strict dominance.

   \noindent\textbf{Existence of $s^*$} When $\eta(s^{CART}, t) > c$, we have $(\eta(s^{CART}) - c)(\mu_R(s^{CART}) - \mu_L(s^{CART})) > 0$, so we want to show that $\exists\ s \in (0,s^{CART})$ such that $\eta(s) = c$.

    \begin{equation*}
        \mu_L = {\rm min}(\mu_L,\mu_R) \leq c \implies \exists\ \tilde{s}\in [0,s^{CART}] \text{ such that } \eta(\tilde{s}) \leq c
    \end{equation*}
    If $\eta(\tilde{s}) = c$, then the existence of $s^*$ is proven. If $\eta(\tilde{s}) < c$, then by intermediate value theorem, $\exists\ s\in (\tilde{s},s^{CART})$ such that $\eta(s)= c$, the proof for existence of $s^*$ is complete for $\eta(s^{CART}, t) > c$. The case when $\eta(s^{CART}, t) \leq c$ can be proved using the same logic.
    
    We showed that when $\eta(s^{CART}, t) > c$, there exists $s^*$ and the $s^*$ is in the range $(0,s^{CART})$. By continuity of $\eta$ and the definition of $s^*$, $\forall\ s \in (s^*,s^{CART}), \eta(s) > c$. By the same logic, we have that when $\eta(s^{CART}, t) \leq c$, $\forall\ s \in (s^{CART},s^*), \eta(s) \leq c$.

    \textbf{Strict dominance} We compare risks led by $s^{CART}$ and $s^*$:\\
    1. If $\eta(s^{CART}) \leq c$, for any $s\in (s^{CART}, s^*)$,  we have $\eta(s) \leq c$ and 
    \begin{align*}
    \mu_L(s) &= \int_{0}^s \eta(x) f(x) dx / F(s) = \frac{\int_{0}^{s^{CART}} \eta(x) f(x) dx + \int_{s^{CART}}^{s^{*}} \eta(x) f(x) dx}{F(s)} \\
    &< \frac{\int_{0}^{s^{CART}} \eta(x) f(x) dx + c\int_{s^{CART}}^{s^{*}}  f(x) dx}{F(s)} = \frac{\mu_L(s^{CART}) F(s^{CART}) + c(F(s) - F(s^{CART}))}{F(s)} < c\\
    \mu_R(s) &> \frac{\mu_R(s^{CART}) (1-F(s^{CART})) - c(F(s) - F(s^{CART}))}{1-F(s)} > c.
    \end{align*}

    Since the order that $\mu_L(s) \leq c < \mu_R(s)$ does not change and $\forall\ s \in (s^{CART},s^*), \eta(s) \leq c$, we have that when $X \in [0,s^{CART}]$, splitting rules $s$ and $s^{CART}$ are equivalent in classifying the latent probability; when $X \in (s^{CART},s)$ which has a nonzero measure, splitting rule $s$ performs strictly better than $s^{CART}$  in classifying the latent probability; when $X \in [s,1]$, splitting rules $s$ and $s^{CART}$ are equivalent in classifying the latent probability. Therefore, $s$ strictly dominates $s^{CART}$. \\
    2. If $\eta(s^{CART}) > c$, for $s\in (s^*, s^{CART})$,  we have $\eta(s) > c$, $s$ strictly dominates $s^{CART}$.
\end{proof}

{\bf Theoretical property of KD-CART}

To explore the theoretical property of KD-CART, we treat it as a CART that takes true $\eta(x)$ as input and splits the population based on $\eta(x)$. KD-CART uses criterion function $\mathcal{G}^{KD} := F_X(s) {\rm Var}(\eta(X)|X\leq s) + (1-F_X(s)) {\rm Var}(\eta(X)|X> s)$. 

\begin{lemma} \label{lemma mid point for KD}
    Define $s^{KD} := \argmin_{s\in (0,1)} \mathcal{G}^{KD}(s)$. $2\eta(s^{KD}) - \mu_{L}(s^{KD}) - \mu_{R}(s^{KD}) = 0$ and $\mu_L(s^{KD}) \neq \mu_R(s^{KD})$.
\end{lemma}

The proof of Lemma \ref{lemma mid point for KD} is almost identical to Lemma \ref{lemma mid point} and is omitted here. The key take away is Theorem \ref{theorem tendency} also applies to $s^{KD}$, as proof of Theorem \ref{theorem tendency} only requires that $\eta(s^{CART})$ is strictly between $\mu_L$ and $\mu_R$, Lemma \ref{lemma mid point for KD} shows that $\eta(s^{KD})$ satisfies this condition. As a result, KD-CART does not minimize the misclassification risk.

\textbf{Explanation for discrete $X$}

Say $X \in \{a_1, a_2, \dots, a_K\}$ where $a_1, a_2, \dots, a_K$ are ordered and $a_1 < a_2 < \dots < a_K$. Assumption \ref{assumption unique intersection between eta and c} (a discrete $X$ version) would guarantee that for one unique split rule, for all possible values of $X$ in the left node, $\eta(X) > c$ and for all possible values of $X$ in the right node, $\eta(X) < c$, or vice versa. We can identify this unique rule by simply comparing $\eta(a_k)$ and $c$ for $k = 1, 2, \dots, K$. Estimation boils down to estimating $\eta(X)$ for finitely many $X$, which is trivial.

We find it clearer to prove Corollary \ref{theorem monotonic eta uniform X} and then plug in some of the steps in the proof for Corollary \ref{theorem monotonic eta uniform X} into proof for Theorem \ref{theorem unique intersection between eta and c}. Hence, we first present the proof for Corollary \ref{theorem monotonic eta uniform X}.

\begin{corollary}\label{theorem monotonic eta uniform X}
    If $X \sim {\rm Unif}[0,1]$, and $\eta(X)$ is monotonic $\forall~X \in [0,1]$ and is strictly monotonic and differentiable in a neighborhood of $s^*$, 
    then $\argmax_s \mathcal{G}^*(s,c)$ 
    identifies $s^*$.
\end{corollary} 

{\bf Proof for Corollary \ref{theorem monotonic eta uniform X}}
\begin{proof}
    WLOG, assume $\eta$ is monotonically increasing with respect to x. The proof for the monotonically decreasing case is similar. We will first discuss two distinct cases: $\mu_L < \mu_R \leq c$ and $c \leq \mu_L < \mu_R$ and show that for both cases $\mathcal{G}^*$ is constant. Then, we consider the case where $\mu_L \leq c \leq \mu_R$
    and show that
    $s^*$ maximizes $\mathcal{G}^*$ in this case. We complete the proof by showing that $\mathcal{G}^*(s^*,c)$ is larger than the two constant $\mathcal{G}^*$ for the first two cases.

    i. $\mu_L < \mu_R \leq c$, we can write $\mathcal{G}^*$ as 
    \begin{align*}
        s\left(c-\frac{\int_0^s \eta(x) dx}{s}\right) + (1-s)\left(c-\frac{\int_s^1 \eta(x) dx}{1-s}\right) = c - \int_0^1 \eta(x) dx
    \end{align*}
    Note that in this case, $\mathcal{G}^*$ is a constant.
    
    ii. $ c \leq \mu_L < \mu_R$, we can write $\mathcal{G}^*$ as 
    \begin{align*}
        s\left(\frac{\int_0^s \eta(x) dx}{s} - c\right) + (1-s)\left(\frac{\int_s^1 \eta(x) dx}{1-s} - c\right) = \int_0^1 \eta(x) dx - c
    \end{align*}
    Note that in this case, $\mathcal{G}^*$ is a constant.

    iii. $\mu_L \leq c \leq \mu_R$, we can write $\max_s \mathcal{G}^*$ as 
    \begin{align*}
        \max_s s\left(c - \frac{\int_0^s \eta(x) dx}{s} \right) + (1-s)\left( \frac{\int_s^1 \eta(x) dx}{1-s} - c \right) 
    \end{align*}

    The first-order condition is $c - \eta(s) - \eta(s) + c = 0$. We solve $\eta(s) = c$, hence, $s^*$ is a local optima. Moreover, second order derivative is $-2\eta'(s^*)$ which is negative given that $\eta$ is strictly increasing and differentiable in a neighborhood of $s^*$. Hence, $s^*$ is a local maxima. Since $\mathcal{G}^*$ is continuous in $s$, we need to show that $\mathcal{G}^*(s^*)$ is larger than the boundary points for the case $\mu_L \leq c \leq \mu_R$ to claim $s^*$ as the global maxima for case iii. There are two possibilities for the boundary points

    Possibility 1:
    Consider $s_1$ and $s_2$ and their associated $\mathcal{G}^*(s_1)$ and $\mathcal{G}^*(s_2)$ where
    $\mu_L(s_1) = c$ and $\mu_R(s_2) = c$. Since $s_1$ is included in case ii; whereas $s_2$ is included in case i, once we show that $\mathcal{G}^*(s^*) > \mathcal{G}^*(s_1), \mathcal{G}^*(s_2)$, we can claim $s^*$ to be the unique global maxima among all possible $s$.

    \begin{align*}
        \mathcal{G}^*(s^*) =& s^*c - \int_0^{s^*} \eta(x) dx  + \int_{s^*}^1 \eta(x) dx   - (1 - s^*)c \\
        >& s^*c - \int_0^{s^*} \eta(x) dx + (1-s^*)c - \int_{s^*}^1 \eta(x) dx \\
        =& c - \int_0^1 \eta(x) dx = \mathcal{G}^*(s_2) \\
         \mathcal{G}^*(s^*) =& s^*c - \int_0^{s^*} \eta(x) dx  + \int_{s^*}^1 \eta(x) dx   - (1 - s^*)c \\
        >& \int_0^{s^*} \eta(x) dx - s^*c + (1-s^*)c - (1 - s^*)c \\
        =& \int_0^1 \eta(x) dx - c = \mathcal{G}^*(s_1)
    \end{align*}

Possibility 2: Even at the boundary point,  $\mu_L \leq c \leq \mu_R$ still holds. Then, the boundary points for case iii are $s=0$ and $s=1$.
\begin{align*}
    \tilde{G}(0) =& \int_0^1 \eta(x) dx - c < \mathcal{G}^*(s^*) \\
    \tilde{G}(1) =& c - \int_0^1 \eta(x) dx < \mathcal{G}^*(s^*)
\end{align*}
Since under Possibility 2 case iii spans the entire domain of $X$, $s^*$ is the unique global maxima. Note that Assumption \ref{assumption unique intersection between eta and c} implies that $\exists~ \epsilon > 0$, such that $s^* \in [\epsilon,1-\epsilon] \subset [0,1]$. Therefore, $s^*$ is the global maxima among all possible $s$. 
\end{proof}
\vspace{-1ex}
{\bf Proof for Theorem \ref{theorem unique intersection between eta and c}}
\begin{proof}
    Again, assume that $\eta$ is monotonically increasing in the neighborhood of $s^*$, the proof for the monotonically decreasing case is the same. Assumption \ref{assumption unique intersection between eta and c} guarantees $\forall s < s^*, \eta(X) < c$, $\forall s > s^*, \eta(X) > c$. 

    Let $[s^*-\epsilon_1,s^*+\epsilon_1]$ denote the monotonically increasing $\eta$ neighborhood of $s^*$, where $\epsilon_1 > 0$. Let $[s^*-\epsilon_2,s^*+\epsilon_2]$ denote an interval that contains $s^*$ in which $\forall s \in [s^*-\epsilon_2,s^*+\epsilon_2], \mu_L < c < \mu_R$ and $\epsilon_2 > 0$. The existence of $[s^*-\epsilon_2,s^*+\epsilon_2]$ is guaranteed because both $\mu_L$ and $\mu_R$ are continuous in $s$ and given Assumption \ref{assumption unique intersection between eta and c}, $\mu_L(s^*)<c<\mu_R(s^*)$. 

    Consider $s \in  [s^*-\epsilon_1,s^*+\epsilon_1] \cap [s^*-\epsilon_2,s^*+\epsilon_2] = [a,b]$, where $a = \max(s^*-\epsilon_1,s^*-\epsilon_2)$ and $b = \min(s^*+\epsilon_1,s^*+\epsilon_2)$. 
    This case is equivalent to case iii in the proof for Corollary \ref{theorem monotonic eta uniform X}. Hence, $s^*$ is a local maxima for the interval $[a,b]$ and $\mathcal{G}^*$ is
    \[
    s^*c - \int_0^{s^*} \eta(x) dx + \int_{s^*}^1 \eta(x) dx - (1 - s^*)c.
    \]

    Consider $s \leq a$. When $\mu_R \leq c$, $\mathcal{G}^*$ is $c - \int_0^1 \eta(x) dx$, the proof is identical to case i in the proof for Corollary \ref{theorem monotonic eta uniform X}. When $\mu_R > c$, 
    \begin{align*}
        \mathcal{G}^* =& sc - \int_0^{s} \eta(x) dx + \int_s^1 \eta(x) dx - (1-s)c \\
        =& sc + \int_s^{s^*} \eta(x) dx - \int_0^{s^*} \eta(x) dx + \int_{s^*}^1 \eta(x) dx + \int_{s}^{s^*} \eta(x) dx - (1-s)c\\
        <& sc + (s^* -s)c - \int_0^{s^*} \eta(x) dx + \int_{s^*}^1 \eta(x) dx + (s^* -s)c - (1-s)c \\
        =& s^*c - \int_0^{s^*} \eta(x) dx + \int_{s^*}^1 \eta(x) dx - (1 - s^*)c
    \end{align*}

    The proof for when $s \geq b$ is similar, those $s$ results in higher $\mathcal{G}^*$.

    Since $s^*$ is the unique point in the interval $[a,b]$ which meets the first-order condition and the boundary points $\{a,b\}$ have lower $\mathcal{G}^*$ than $s^*$, $s^*$ is the unique maxima in the interval $[a,b]$. Moreover, $s \in [0,a] \cup [b,1]$ also have lower $\mathcal{G}^*$ than $s^*$, hence, $s^*$ is the unique global maxima.
\end{proof}
{\bf Proof for Theorem \ref{theorem MDFS consistency}}
\begin{proof}
\textbf{Outline}: We leverage tools for studying asymptotic properties of M-estimator. To show $\hat{s} \stackrel{p}{\to} s^*$, we take two steps:
\begin{enumerate}
    \item Under Assumption \ref{assumption unique intersection between eta and c}, we show uniform convergence of the estimator of cost function $\widehat{{\cal G}}^*$ to its population target ${\cal G}^*$, i.e., $\sup_{s \in (\epsilon,1-\epsilon)} | {\cal G}^*(s,c) - \widehat{{\cal G}}^*(s,c)|  \stackrel{p}{\to} 0$ as $n\to\infty$.
    \item Combining with Theorem \ref{theorem unique intersection between eta and c}, we can apply Theorem 5.7 of \citet{van2000asymptotic} and get the desired result.
\end{enumerate}
{\textbf {Step 1}}. For any given $c \in (0,1)$ and $\epsilon_0>0$, we have
    \begin{align}
    &~\pr\Bigg(
    \sup_{s \in (\epsilon,1-\epsilon)} | {\cal G}^*(s,c) - \widehat{{\cal G}}^*(s,c)| \geq \epsilon_0
    \Bigg)
    \notag\\
    =&~
    \pr\Bigg(
    \sup_{s \in (\epsilon,1-\epsilon)}
    \Bigg|
    s
    \bigg|
        \frac
        {\sum_{i=1}^{n} Y_i\mathbbm{1}\{X_i \leq s\}}
        {\sum_{i=1}^{n}\mathbbm{1}\{X_i \leq s\}}
        -c
    \bigg| - s|\mu_L-c| 
    \notag\\
    + &~
    (1-s)
    \bigg|
        \frac
        {\sum_{i=1}^{n} Y_i \mathbbm{1}\{X_i > s\}}
        {\sum_{i=1}^{n}\mathbbm{1}\{X_i > s\}}
        -c
    \bigg| - (1-s)|\mu_R - c|
    \Bigg| \geq \epsilon_0
    \Bigg)
    \notag\\
    \label{proofeqn::consistency-part1}
    \leq&~
    \pr\Bigg( \sup_{s \in (\epsilon,1-\epsilon)} 
    s\Bigg|
    \frac
    {\sum_{i=1}^{n} Y_i\mathbbm{1}\{X_i \leq s\}}
    {\sum_{i=1}^{n}\mathbbm{1}\{X_i \leq s\}}
    - \mu_L\Bigg| \geq \epsilon_0/2 \Bigg) 
    \\
    \label{proofeqn::consistency-part2}
    + &~
    \pr\Bigg( \sup_{s \in (\epsilon,1-\epsilon)} 
    (1-s)\Bigg|
    \frac
        {\sum_{i=1}^{n} Y_i \mathbbm{1}\{X_i > s\}}
        {\sum_{i=1}^{n}\mathbbm{1}\{X_i > s\}}
    - \mu_R\Bigg|\geq \epsilon_0/2 \Bigg)
\end{align}
We focus on showing \eqref{proofeqn::consistency-part1} goes to 0 as $n\to\infty$, the proof for \eqref{proofeqn::consistency-part2} is exactly the same.

For any $\delta > 0$, denote event ${\cal E} = \inf_{s \in (\epsilon,1-\epsilon)} n^{-1}\sum_{i=1}^{n} \mathbbm{1}_{\{X_i \leq s\}} \geq \delta$. We have
\begin{align}
    &~\pr\Bigg( \sup_{s \in (\epsilon,1-\epsilon)} 
    s\Bigg|
    \frac
    {\sum_{i=1}^{n} Y_i\mathbbm{1}\{X_i \leq s\}}
    {\sum_{i=1}^{n}\mathbbm{1}\{X_i \leq s\}}
    - \mu_L\Bigg| \geq \epsilon_0/2 \Bigg)
    \notag\\
    \leq&~
    \pr\Bigg( \sup_{s \in (\epsilon,1-\epsilon)}
    s\Bigg(
        \frac
        {\big|n^{-1}\sum_{i=1}^{n} Y_i\mathbbm{1}\{X_i \leq s\} - s\mu_L\big|}
        {n^{-1}\sum_{i=1}^{n}\mathbbm{1}\{X_i \leq s\}}
        +
        \frac
        {\mu_L\big|n^{-1}\sum_{i=1}^{n}\mathbbm{1}\{X_i \leq s\}-s\big|}
        {n^{-1}\sum_{i=1}^{n}\mathbbm{1}\{X_i \leq s\}}
     \Bigg) \geq \epsilon_0/2   
    \Bigg)
    \notag\\
    \leq&~
    \pr
    \Bigg( \sup_{s \in (\epsilon,1-\epsilon)}
    \frac
        {s\big|n^{-1}\sum_{i=1}^{n} Y_i\mathbbm{1}\{X_i \leq s\} - s\mu_L\big|}
        {n^{-1}\sum_{i=1}^{n}\mathbbm{1}\{X_i \leq s\}}
        \geq \epsilon_0/4
    \Bigg) + 
    \pr
    \Bigg( \sup_{s \in (\epsilon,1-\epsilon)}
        \frac
        {\mu_L\big|n^{-1}\sum_{i=1}^{n}\mathbbm{1}\{X_i \leq s\}-s\big|}
        {n^{-1}\sum_{i=1}^{n}\mathbbm{1}\{X_i \leq s\}}
        \geq \epsilon_0/4
    \Bigg)
    \notag\\
    \leq&~
    \pr
    \Bigg(
    \frac
        {\sup_{s \in (\epsilon,1-\epsilon)}
        s\big|n^{-1}\sum_{i=1}^{n} Y_i\mathbbm{1}\{X_i \leq s\} - s\mu_L\big|}
        {\inf_{s \in (\epsilon,1-\epsilon)} n^{-1}\sum_{i=1}^{n}\mathbbm{1}\{X_i \leq s\}}
        \geq \epsilon_0/4
    \Bigg) + 
    \pr
    \Bigg(
        \frac
        {\sup_{s \in (\epsilon,1-\epsilon)}
        \mu_L\big|n^{-1}\sum_{i=1}^{n}\mathbbm{1}\{X_i \leq s\}-s\big|}
        {\inf_{s \in (\epsilon,1-\epsilon)} n^{-1}\sum_{i=1}^{n}\mathbbm{1}\{X_i \leq s\}}
        \geq \epsilon_0/4
    \Bigg)
    \notag\\
    \leq&~
    \pr
    \Bigg( 
        \frac
        {\sup_{s \in (\epsilon,1-\epsilon)}
        s\big|n^{-1}\sum_{i=1}^{n} Y_i\mathbbm{1}\{X_i \leq s\} - s\mu_L\big|}
        {\inf_{s \in (\epsilon,1-\epsilon)} n^{-1}\sum_{i=1}^{n}\mathbbm{1}\{X_i \leq s\}}
        \geq \epsilon_0/4,
        {\cal E}
    \Bigg)
    +
    \notag\\
    &~\pr
    \Bigg( 
        \frac
        {\sup_{s \in (\epsilon,1-\epsilon)}
        s\big|n^{-1}\sum_{i=1}^{n} Y_i\mathbbm{1}\{X_i \leq s\} - s\mu_L\big|}
        {\inf_{s \in (\epsilon,1-\epsilon)} n^{-1}\sum_{i=1}^{n}\mathbbm{1}\{X_i \leq s\}}
        \geq \epsilon_0/4,
        {\cal E}^c
    \Bigg) +
    \notag\\
    &~\pr
    \Bigg(
        \frac
        {\sup_{s \in (\epsilon,1-\epsilon)}
        \mu_L\big|n^{-1}\sum_{i=1}^{n}\mathbbm{1}\{X_i \leq s\}-s\big|}
        {\inf_{s \in (\epsilon,1-\epsilon)} n^{-1}\sum_{i=1}^{n}\mathbbm{1}\{X_i \leq s\}}
        \geq \epsilon_0/4,
        {\cal E}
    \Bigg) + 
    \notag\\
    \label{proofeqn::consistency-LOTP}
    &~\pr
    \Bigg(
        \frac
        {\sup_{s \in (\epsilon,1-\epsilon)}
        \mu_L\big|n^{-1}\sum_{i=1}^{n}\mathbbm{1}\{X_i \leq s\}-s\big|}
        {\inf_{s \in (\epsilon,1-\epsilon)} n^{-1}\sum_{i=1}^{n}\mathbbm{1}\{X_i \leq s\}}
        \geq \epsilon_0/4,
        {\cal E}^c
    \Bigg)
    \\
    \label{proofeqn::consistency-numerator}
    \leq&~
     \pr
    \Bigg( 
    \sup_{s \in (\epsilon,1-\epsilon)}
        s\Big|n^{-1}\sum_{i=1}^{n} Y_i\mathbbm{1}\{X_i \leq s\} - s\mu_L\Big| \geq (\epsilon_0\delta)/4
    \Bigg) + 
    \pr({\cal E}^c)
    \\
    \label{proofeqn::consistency-denominator}
    +&~
    \pr
    \Bigg(
        \sup_{s \in (\epsilon,1-\epsilon)}
        \mu_L\Big|n^{-1}\sum_{i=1}^{n}\mathbbm{1}\{X_i \leq s\}-s\Big|
        \geq (\epsilon_0\delta)/4
    \Bigg) +
    \pr({\cal E}^c),
\end{align}
where \eqref{proofeqn::consistency-LOTP} is by law of total probability.
For the first term in \eqref{proofeqn::consistency-numerator}, let $\bm{Z}_i = (X_i,Y_i), \,i=1,\dots,n$ and $f(\bm{Z}_i,s) = Y_i\mathbbm{1}_{\{X_i \leq s\}}$. Notice $\ep[f(\bm{Z},s)] = s\mu_L$, $f(\bm{Z},s)$ is continuous at each $s \in (0,1)$ for almost all $\bm{Z}$, since discontinuity occurs at $X_i = s$, which has measure zero for $X_i \sim \textrm{Unif}(0,1)$. Also, $f(\bm{Z},s) \leq \mathbbm{1}_{\{0 \leq X \leq 1\}}$ with $\ep[\mathbbm{1}_{\{0 \leq X \leq 1\}}] = 1 \leq \infty$. Applying uniform law of large numbers:
\begin{align}
    &~\pr
    \Bigg( 
    \sup_{s \in (\epsilon,1-\epsilon)}
        s\big|n^{-1}\sum_{i=1}^{n} Y_i\mathbbm{1}\{X_i \leq s\} - s\mu_L\big| \geq (\epsilon_0\delta)/4
    \Bigg)
    \notag\\
    \leq&~
    \pr
    \Bigg( 
    \sup_{s \in (\epsilon,1-\epsilon)}
        \big|n^{-1}\sum_{i=1}^{n} Y_i\mathbbm{1}\{X_i \leq s\} - s\mu_L\big| \geq (\epsilon_0\delta)/4
    \Bigg)
    \notag\\
    \leq&~
    \pr
    \Bigg( 
    \sup_{s \in (0,1)}
        \big|n^{-1}\sum_{i=1}^{n} Y_i\mathbbm{1}\{X_i \leq s\} - s\mu_L\big| \geq (\epsilon_0\delta)/4
    \Bigg)
    \to 0,
\end{align}
as $n\to\infty$. For the first term in \eqref{proofeqn::consistency-denominator}, notice $F_n(s) = n^{-1} \sum_{i=1}^{n}\mathbbm{1}\{X_i \leq s\}$ is the empirical CDF of $X \sim \textrm{Unif}(0,1)$, with CDF $F_X(s) = s$ for $s \in (0,1)$. Applying Dvoretzky–Kiefer–Wolfowitz inequality, we have
\begin{align}
    &~\pr
    \Bigg(
        \sup_{s \in (\epsilon,1-\epsilon)}
        \mu_L\big|n^{-1}\sum_{i=1}^{n}\mathbbm{1}\{X_i \leq s\}-s\big|
        \geq (\epsilon_0\delta)/4
    \Bigg)
    \notag\\
   (\mu_L \in [0,1]) \leq&~
    \pr
    \Bigg(
        \sup_{s \in (\epsilon,1-\epsilon)}
        \big|n^{-1}\sum_{i=1}^{n}\mathbbm{1}\{X_i \leq s\}-s\big|
        \geq (\epsilon_0\delta)/4
    \Bigg)
    \notag\\
    \leq&~
    \pr
    \Bigg(
        \sup_{s \in (0,1)}
        \big|n^{-1}\sum_{i=1}^{n}\mathbbm{1}\{X_i \leq s\}-s\big|
        \geq (\epsilon_0\delta)/4
    \Bigg)\to 0,
\end{align}
as $n\to\infty$. 

For the second terms in \eqref{proofeqn::consistency-numerator} and \eqref{proofeqn::consistency-denominator}, we want to show $\pr({\cal E}^c) \to 0$ as $n\to\infty$. By Glivenko–Cantelli theorem, we have $\sup_{s \in (0,1)} |F_n(s) - F_X(s)| \stackrel{a.s.}{\to} 0$, which implies $\sup_{s \in (\epsilon,1-\epsilon)} |F_n(s) - s| \stackrel{a.s.}{\to} 0$. The uniform convergence means for all $s \in (\epsilon,1-\epsilon)$ and any $\epsilon'>0$, there exists $N$ such that for all $n\geq N$, we have $|F_n(s) - s| < \epsilon'$, which implies $|\inf_{s \in (\epsilon,1-\epsilon)} F_n(s) - \inf_{s \in (\epsilon,1-\epsilon)} s| = |\inf_{s \in (\epsilon,1-\epsilon)} F_n(s) - \epsilon| < \epsilon'$. This shows $\inf_{s \in (\epsilon,1-\epsilon)} F_n(s)  \stackrel{p}{\to} \epsilon$. Finally, setting $\delta = \epsilon/2$, we have $\inf_{s \in (\epsilon,1-\epsilon)} F_n(s) - \delta \stackrel{p}{\to} \epsilon/2$. Hence $\pr({\cal E}^c) = \pr(\inf_{s \in (\epsilon,1-\epsilon)} F_n(s) - \delta < 0) \to 0$ as $n\to\infty$.

Combining \eqref{proofeqn::consistency-numerator} and \eqref{proofeqn::consistency-denominator}, we can show \eqref{proofeqn::consistency-part1} goes to 0 as $n\to\infty$. Similarly, \eqref{proofeqn::consistency-part2} goes to 0 as $n\to\infty$. This gives us 
$\pr(\sup_{s \in (\epsilon,1-\epsilon)} | {\cal G}^*(s) - \widehat{{\cal G}}^*(s)| \geq \epsilon_0) \to 0$ as $n\to\infty$, which completes the proof of Step 1.

\textbf{Step 2}. By Theorem \ref{theorem unique intersection between eta and c}, we have for any $\epsilon > 0$, $\sup_{s: d(s,s^*) > \epsilon} {\cal G}^*(s,c) < {\cal G}^*(s^*,c)$. Combining with $\sup_{s \in (\epsilon,1-\epsilon)} | {\cal G}^*(s,c) - \widehat{{\cal G}}^*(s,c)|  \stackrel{p}{\to} 0$, we apply Theorem 5.7 of \citet{van2000asymptotic} and conclude $\hat{s} \stackrel{p}{\to} s^*$.

\end{proof}

\textbf{Intuition for Theorem \ref{theorem penalized loss works}}

\begin{wrapfigure}{r}{0.5\textwidth}
    \vspace{-15pt}
     \includegraphics[width=\linewidth]{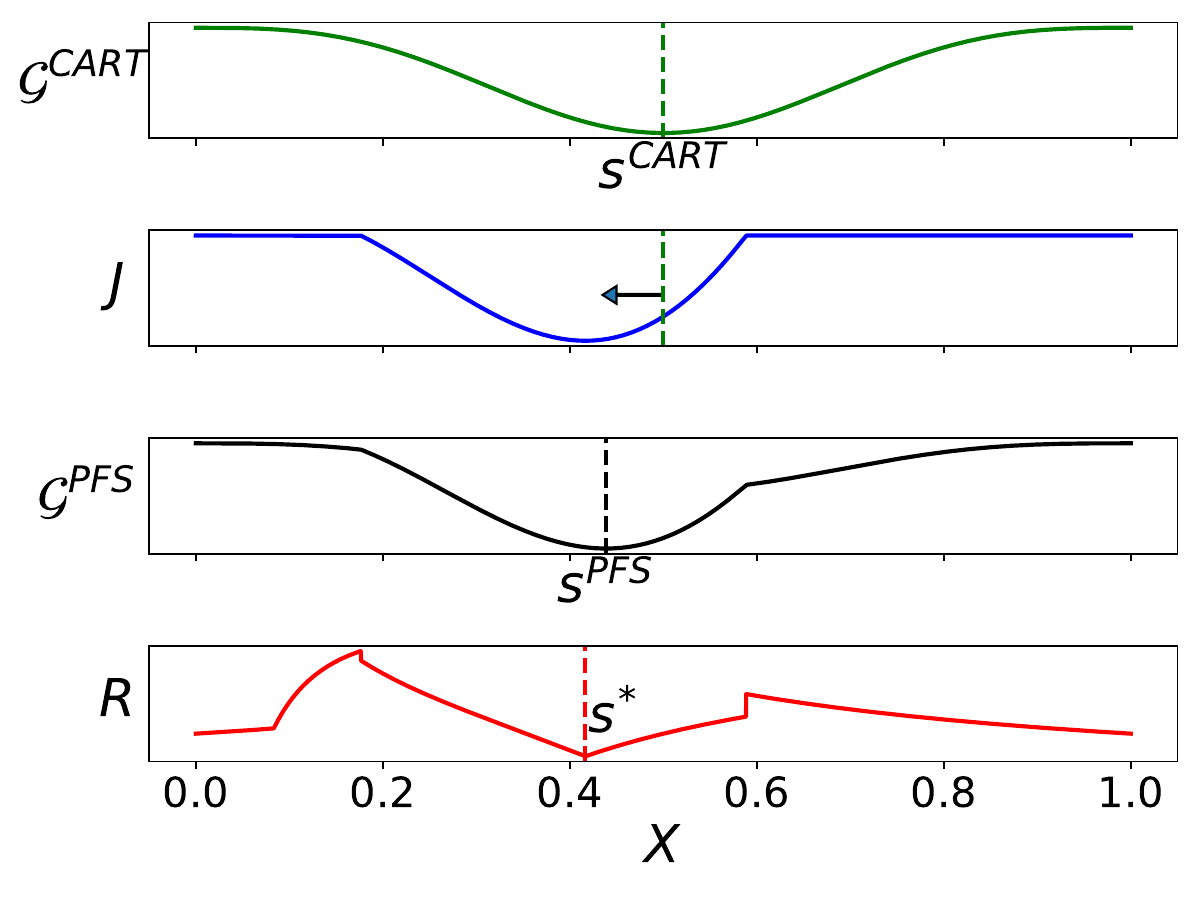}
        \caption{Visualization of the different loss components of PFS. 
    }
   \label{fig proposed solutions losses}
    \vspace{-12pt}
\end{wrapfigure}

We use Figure \ref{fig proposed solutions losses} to illustrate the intuition behind Theorem \ref{theorem penalized loss works} based on the same setting as that in Figure \ref{Figure Sine example}. From top to bottom, Figure \ref{fig proposed solutions losses} depicts $\mathcal{G}^{CART}$, penalty term $J$, $\mathcal{G}^{PFS}$ and LPC misclassification risk $R$ as functions of feature $X$.
The green curve shows that $\mathcal{G}^{CART}$ is minimized at $s^{CART}$. The green curve is also very \textit{flat} around $s^{CART}$. In contrast, the blue curve shows that penalty $J$ increases \textit{sharply} at $X=s^{CART}$. Minimizing $J$ would shift the split rule from $s^{CART}$ to the left, as indicated by the arrow. Since $\mathcal{G}^{PFS} = \mathcal{G}^{CART} + \lambda J$, the gradient of $\mathcal{G}^{PFS}$ at $s^{CART}$ is largely determined by the gradient of $J$, given how flat $\mathcal{G}^{CART}$ is at $s^{CART}$ (See the black curve). Hence, minimizing $\mathcal{G}^{PFS}$ also shifts the split rule from $s^{CART}$ to the left (i.e., towards $s^*$, the minimizer of the LPC misclassification risk as shown by the red curve). Our formal proof in Appendix \ref{appendix proofs} shows that the \textit{flatness} of $\mathcal{G}^{PFS}$ and \textit{sharp} change in $J$ at $X=s^{CART}$ and shifting $s^{CART}$ towards $s^*$ are not coincidences; they are generally true under the assumptions in Theorem \ref{theorem penalized loss works}.

{\bf Proof for Theorem \ref{theorem penalized loss works}}
\begin{proof}
    Consider 
    \begin{align*}
        \frac{\partial \mathcal{G}^{PFS}(s, c)}{\partial s} = \frac{\partial \mathcal{G}^{CART}(s, c) + \lambda W(|\mu_L - c|)F(s) + \lambda W(|\mu_R - c|)(1-F(s))}{\partial s}.
    \end{align*}
    We focus on the new term $\frac{\partial W(|\mu_L - c|)F(s)}{\partial s}$ and $\frac{\partial W(|\mu_R - c|)(1-F(s))}{\partial s}$. By \eqref{dElds} and \eqref{dErds},
    \begin{align*}
        \frac{\partial W(|\mu_L - c|)F(s)}{\partial s} =&\ f(s)W(|\mu_L - c|) + \frac{\partial W(|\mu_L - c|)}{\partial \mu_L}\frac{\partial \mu_L}{\partial s}F(s)\\
        =&\ f(s)\left(W(|\mu_L - c|) + \frac{\partial W(|\mu_L - c|)}{\partial \mu_L} (\eta(s) - \mu_L)\right)\\
        \frac{\partial W(|\mu_R - c|)(1-F(s))}{\partial s} =&\ -f(s)W(|\mu_R - c|) + \frac{\partial W(|\mu_R - c|)}{\partial \mu_R}\frac{\partial \mu_R}{\partial s}(1-F(s))\\
        =&\ f(s)\left( - W(|\mu_R - c|) + \frac{\partial W(|\mu_R - c|)}{\partial \mu_R} (\mu_R - \eta(s))\right).
    \end{align*}
    With $\mathcal{G}^{PFS}(s, c)$'s derivative with respect to $s$, we evaluate how $\frac{\partial \mathcal{G}^{PFS}(s, c)}{\partial s}$ behaves at $s^{CART}$, i.e., when $2\eta(s) - \mu_L - \mu_R = 0$ and at $s^*$, i.e., when $\eta(s) = c$ respectively:
    \begin{align}
        \frac{1}{\lambda f(s)}\left.\frac{\partial \mathcal{G}^{PFS}(s, c)}{ \partial s}\right\vert_{s = s^{CART}}  &=\ W(|\mu_L - c|) - W(|\mu_R - c|)  \nonumber\\ 
        &\quad + \left(\frac{\partial W(|\mu_L - c|)}{\partial \mu_L} + \frac{\partial W(|\mu_R - c|)}{\partial \mu_R}\right)(\eta(s^{CART})-\mu_L)\label{gsop}\\
        \frac{1}{f(s)}\left.\frac{\partial \mathcal{G}^{PFS}(s, c)}{\partial s}\right\vert_{s = s^{*}} &=\ (2c - \mu_L - \mu_R)(\mu_R - \mu_L) +  \lambda(W(|\mu_L - c|) - W(|\mu_R - c|)) \nonumber\\
        &\quad + \lambda\left(\frac{\partial W(|\mu_L - c|)}{\partial \mu_L}(c - \mu_L) + \frac{\partial W(|\mu_R - c|)}{\partial \mu_R}(\mu_R - c)\right).\label{gsstar}
    \end{align}
    Note that $s^{*}$ in \eqref{gsstar} is only guaranteed to exist when the condition of $c$ in Theorem~\ref{theorem tendency} is met. We will address this in the latter part of the proof.
    Next, 
    consider all five possible scenarios of $c$'s location separately. (Still, without loss of generality, assume $\mu_L(s^{CART}) < \eta(s^{CART}) < \mu_R(s^{CART})$.)\\
    \newline
    \textbf{i}. $\mu_L(s^{CART}) \le c <\eta(s^{CART}) < \mu_R(s^{CART})$. 
    
    By Lemma~\ref{lemma mid point},$|\mu_R(s^{CART}) - c| > |\mu_L(s^{CART}) - c|$. $W(|\mu_L - c|) > W(|\mu_R - c|)$ by the monotonicity of $W$ and $\frac{\partial W(|\mu_L - c|)}{\partial \mu_L} + \frac{\partial W(|\mu_R - c|)}{\partial \mu_R} \ge 0$ by the convexity of $W$. 
    Thus, by \eqref{gsop}, 
    \begin{align}
    \left.\frac{\partial \mathcal{G}^{PFS}(s, c)}{\partial s}\right\vert_{s = s^{CART}} >0.    \label{greater0}
    \end{align}
    By Lemma~\ref{lemma mid point}, and continuity of $\eta$, $\mu_L$ and $\mu_R$, 
    \begin{align*}
        (2\eta(s) - \mu_L - \mu_R)(\mu_R - \mu_L) < 0,\ s\in (s^{sp}, s^{CART}), 
    \end{align*}
    where $s^{sp} := \max\{s: s < s^{CART}, (2\eta(s) - \mu_L - \mu_R)(\mu_R - \mu_L) = 0\}$. If there is no such $s^{sp}>0$ exists, let $s^{sp} = 0$.
    When $s\in(\max(s^*, s^{sp}), s^{CART})$, by Theorem~\ref{theorem tendency}, $\mu_L<c$ and $\mu_R>c$. 
    
    When $s^* > s^{sp}$, define $d_L := c - \mu_L$ and $d_R := \mu_R - c$ and we have
    \begin{align*}
    0>(2c - \mu_L(s^*) - \mu_R(s^*))(\mu_R(s^*) - \mu_L(s^*))  = (d_L-d_R)(d_L+d_R),
    \end{align*}
    which directly suggests that $d_L < d_R$.  
    Let $W_2(d) = d^2 + \lambda\left(W(d) - dW'(d)\right)$. We can rewrite \eqref{gsstar} as:
    \begin{align*}
        \frac{1}{f(s)}\left.\frac{\partial \mathcal{G}^{PFS}(s, c)}{\partial s}\right\vert_{s = s^{*}} =&\ d_L^2 + \lambda\left(W(d_L) - d_L\frac{\partial W(d_L)}{\partial d_L}\right) - d_R^2 - \lambda\left(W(d_R) - d_R\frac{\partial W(d_R)}{\partial d_R}\right) \\
        =&\ W_2(d_L) - W_2(d_R).
    \end{align*}
    Note that $W_2^{'}(d) = d(2-\lambda W^{''}(d))$. By the bounded second derivative assumption, set $\Lambda_1 = \frac{1}{\max\{W^{''}(d_L), W^{''}(d_R)\}}$. Since $0<d_L<d_R$, when $\lambda < \Lambda_1$, $W_2(d_L) < W_2(d_R)$, 
    \begin{align*}
        \frac{1}{f(s)}\left.\frac{\partial \mathcal{G}^{PFS}(s, c)}{ \partial s}\right\vert_{s = s^{*}} < 0.
    \end{align*}
    
    When $s^* < s^{sp}$, there exists $s^{in} \in (s^{sp},s^{CART})$ and $c>0$ s.t. 
    \begin{align*}
    (2\eta(s) - \mu_L - \mu_R)(\mu_R - \mu_L) < -c,    
    \end{align*}
    by unique global minimum assumption and continuity. With $\Lambda_1 = \frac{c}{W(0) + W^{'}(0)}$, we have
    \begin{align*}
    \frac{1}{f(s)}\left.\frac{\partial \mathcal{G}^{PFS}(s, c)}{ \partial s}\right\vert_{s = s^{in}} < 0.
    \end{align*}
    
    
    Thus, 
    for $0<\lambda<\Lambda_1$, there exists a $s\in[s^*, s^{CART})$ s.t.
    \begin{align}
    \frac{\partial \mathcal{G}^{PFS}(s, c)}{\partial s} < 0.\label{less0}
    \end{align}
    Considering \eqref{greater0} and \eqref{less0}
    jointly with the continuity of the derivative, we can conclude that there exists a $s^{**}\in (s^*, s^{CART})$ such that $\left.\frac{\partial \mathcal{G}^{PFS}(s, c)}{\partial s}\right\vert_{s = s^{**}} = 0$. $s^{**}$ is a local minimizer of $\mathcal{G}^{PFS}(s, c)$ and is guaranteed to give a smaller risk than $s^{CART}$ by Theorem~\ref{theorem tendency}. 

    Next, we argue that with some additional bound for $\lambda$, $s^{**}$ is also the global optimal split under the penalized impurity measure. To simplify the notation in the rest of the proof, let $s^{in}$ be $s^*$ if $s^*<s^{sp}$. Then,    
    we want to show $\mathcal{G}^{PFS}(s^{**},c) < \mathcal{G}^{PFS}(s,c)$ for $s\neq s^{**}$:
    
    (a). Consider $s$ such that there is at least one $s^{sp}$ in between $s$ and $s^{**}$.
    By unique optimizer assumption for $G$, if there is more than one local minimum, i.e., multiple $s$ that have $2\eta(s, t) - \mu_L - \mu_R = 0$, we assume the global minimum and the second smallest local minimum, gives by $s^{sc}$, different by $\Delta$. Let $\Lambda_2 = \frac{\Delta}{W(0) - \min(W(c), W(1-c))}$. For $0<\lambda<\Lambda_2$,  
    \begin{align*}
    \mathcal{G}^{PFS}(s^{**},c) <&\ \mathcal{G}^{PFS}(s^{CART},c) \\
    =&\ \mathcal{G}(s^{CART},c) + \lambda \left(W(|\mu_L-c|)F(s) + W(|\mu_R-c|)(1-F(s))\right)\\
    \le&\  \mathcal{G}(s^{CART},c) + \lambda W(0) \\
    \le&\ \mathcal{G}(s^{sc},c) + \lambda \min\{W(c), W(1-c)\} \le \mathcal{G}^{PFS}(s,c)
    \end{align*}
    
    (b). Consider $s$ such that there is no $s^{sp}$ in between $s$ and $s^{**}$. 
    
    When $s \in (s^*, s^{CART})$, $\mathcal{G}^{PFS}(s^{**},c) < \mathcal{G}^{PFS}(s,c)$ by definition. 
    
    When $s<s^{*}$, by $\eqref{less0}$ and continuity of $\frac{\partial \mathcal{G}^{PFS}(s, c)}{\partial s}$, there can be two possible scenarios. One, there exist a $s_0<s^*$ such that $\frac{\partial \mathcal{G}^{PFS}(s, c)}{\partial s} = 0$ and $\mathcal{G}^{PFS}(s_0,c) - \mathcal{G}^{PFS}(s^{**},c) := \Delta_L > 0$. Then, define $\Lambda_3 = \frac{\Delta_L}{W(0) - \min(W(c), W(1-c))}$. For $0<\lambda<\Lambda_3$, $\mathcal{G}^{PFS}(s^{**},c) < \mathcal{G}^{PFS}(s,c)$ when $s_0\le s< s^*$ by definition. When $s<s_0$:
    \begin{align*}
    \mathcal{G}^{PFS}(s^{**},c) =&\ \mathcal{G}^{PFS}(s_0,c) - \Delta_L\\
    =&\ \mathcal{G}(s_0,c) + \lambda \left(W(|\mu_L-c|)F(s) + W(|\mu_R-c|)(1-F(s))\right) - \Delta_L\\
    \le&\  \mathcal{G}(s_0,c) + \lambda W(0) - \Delta_L\\
    \le&\ \mathcal{G}(s_0,c) + \lambda \min\{W(c), W(1-c)\}\\
    \le&\ \mathcal{G}(s,c) + \lambda \left(W(|\mu_L-c|)F(s) + W(|\mu_R-c|)(1-F(s))\right) = \mathcal{G}^{PFS}(s,c).
    \end{align*}
    Two, $\frac{\partial \mathcal{G}^{PFS}(s, c)}{\partial s} < 0$. Then $\mathcal{G}^{PFS}(s^{**},c) < \mathcal{G}^{PFS}(s,c)$ for $s<s^{*}$ follows. 

    When $s > s^{CART}$, we can follow the same procedure as when $s<s^{*}$. If there exists a $s_1>s^*$ such that $\frac{\partial \mathcal{G}^{PFS}(s, c)}{\partial s} = 0$ and $\mathcal{G}^{PFS}(s_1,c) - \mathcal{G}^{PFS}(s^{**},c) := \Delta_R > 0$. We can define $\Lambda_4 = \frac{\Delta_R}{W(0) - \min(W(c), W(1-c))}$ and everything else follows.
    
    To conclude, for $0 < \lambda \le \Lambda = \min\{\Lambda_1, \Lambda_2, \Lambda_3, \Lambda_4\}$, the theoretical optimal split chosen with respect to $G^{PFS}$ leads to a risk smaller than $s^{CART}$ when $\mu_L(s^{CART}) \le c <\eta(s^{CART}) < \mu_R(s^{CART})$.\\
    \newline
    \textbf{ii}. $\mu_L(s^{CART}) < \eta(s^{CART}) < c \le \mu_R(s^{CART})$.
    This scenario is symmetric to \textbf{i} and is also considered in Theorem~\ref{theorem tendency}. It can be proved following the same logic as in \textbf{i}.\\
    \newline
    \textbf{iii}. $c < \mu_L(s^{CART}) <\eta(s^{CART}) < \mu_R(s^{CART})$.
    
    In this scenario, consider two split points located on each side of $s^{CART}$ respectively: 
    let $s_{c1}$ be the largest split less than $s^{CART}$ such that $\mu_L(s_{c1}) = c$, $s_{c2}$ be the smallest split greater than $s^{CART}$ such that $\mu_L(s_{c2}) = c$, $s_L$ be the largest split less than $s^{CART}$ such that $\eta(s_L) \ge \mu_L(s^{CART})$ and $s_R$ be the smallest split greater than $s^{CART}$ such that $\eta(s_R) = \mu_R(s^{CART})$. Let $s_{c1} = 0$ or $s_{c2} = 1$ if no such $s_{c1}$ or $s_{c2}$ exists. Both $s_L$ and $s_R$'s existence is guaranteed by the intermediate value theorem which we used once in Theorem~\ref{theorem tendency}. 

    We argue for any $s\in (\max(s_L, s_{c1}), \min(s_R,s_{c2}))$, risk is no greater than $s^{CART}$. 
    
    For $s\in (\max(s_L, s_{c1}), s^{CART})$, it is a piece of $x$ that has $\eta(x) > \mu_L(s^{CART})$ being split to right.
    \begin{align*}
        \mu_R(s) >&\ \frac{\mu_R(s^{CART})(1-F(s^{CART})) + \mu_L(s^{CART})(F(s^{CART})-F(s)) }{1-F(s)} > \mu_L(s^{CART})>c\\
        c < \mu_L(s) <&\ \frac{\mu_L(s^{CART})F(s^{CART}) - \mu_L(s^{CART})(F(s^{CART})-F(s)) }{F(s)} = \mu_L(s^{CART}).
    \end{align*}
    Since $c < \mu_L(s) <\mu_L(s^{CART}) < \mu_R(s)$, the expectations on both sides remain greater than $c$ and the risk is the same as split $s^{CART}$. When $s\in (s^{CART}, \min(s_R,s_{c2}))$, $c < \mu_L(s) <\mu_R(s^{CART}) < \mu_R(s)$ can be obtained similarly and the risk also remains. 

    Note that $\mu_R(s) > \mu_L(s)$ all the way when $s\in (\max(s_L, s_{c1}), \min(s_R,s_{c2}))$. We can then define $\Delta$ the same as in scenario \textbf{i}, $\Delta_L = \mathcal{G}^{PFS}(\max(s_L, s_{c1}), c)$, and $\Delta_R = \mathcal{G}^{PFS}(\min(s_R,s_{c2}), c)$. $\Lambda_2, \Lambda_3, \Lambda_4$ are defined accordingly. 
    For $0<\lambda\le \Lambda =\min\{\Lambda_2, \Lambda_3, \Lambda_4\}$, the theoretical optimal split chosen with respect to $G^{PFS}$ leads to the same risk as $s^{CART}$ when $c < \mu_L(s^{CART}) <\eta(s^{CART}) < \mu_R(s^{CART})$.\\
    \newline 
    \textbf{iv}. $\mu_L(s^{CART}) <\eta(s^{CART}) < \mu_R(s^{CART}) < c$. This scenario is symmetric to \textbf{iii}. It can be proved following the same logic in \textbf{iii}.\\
    \textbf{v}. $\mu_L(s^{CART}) <\eta(s^{CART}) = c < \mu_R(s^{CART})$. This scenario can be considered a special case of \textbf{i} and \textbf{ii}. By appropriate choice of $\Lambda$, it's trivial to show that the original $s^{CART}$ remains the global optimizer and the risk stays the same. 

\end{proof}

{\bf Mathematical details for Figure \ref{Figure Unique Sine example} and \ref{Figure Monotonic example}}

Imagine there are two final splitting nodes $\{t_1,t_2\}$ with the same mass of population. CART selects features $\{X_1, X_2\}$ for nodes $\{t_1,t_2\}$, respectively. The pdf of $X_1$ and $X_2$ in the $\{t_1,t_2\}$ are both uniform, respectively: $f_1(x) = 1$ and $f_2(x) = 1$ $\forall~x \in [0,1]$.

In node $t_1$,
\[
P(Y=1|X_1) = : \eta_1(X_1) = 
\begin{cases}
    1, \text{ when $X_1 \in [0,\frac{1}{4}]$} \\
    \frac{{\rm sin}(2 \pi X_1) + 1}{2} , \text{ when $X_1 \in (\frac{1}{4},\frac{3}{4})$} \\
    0, \text{ when $X_1 \in [\frac{3}{4},1]$}
\end{cases}
\]

We depict $\eta_1(X_1)$ in Figure \ref{Figure Unique Sine example}. Since $\eta_1(X_1)$ is reflectional symmetric around $\eta_1(0.5)$, $s^{cart} = 0.5$. Splitting $s^{cart}$, policymakers target the left node $X_1 \leq 0.5$ only, the calculation is similar to Figure \ref{Figure Sine example} in the introduction.

The unique optimal LPC solution is $s^* = \frac{5}{12}$. By shifting from $s^{CART}$ to $s^*$, the policymaker excludes subpopulation $\{\frac{5}{12}<X_1<\frac{1}{2}|t_1\}$ from the targeted group. Again, the mathematics is similar to Figure \ref{Figure Sine example} in the introduction.

In node $t_2$, 
\[
P(Y=1|X_2) =: \eta_2(X_2) = \frac{9}{11}X_2
\]

The other node $t_2$ is depicted by Figure \ref{Figure Monotonic example}. CART will split $t_2$ at $s^{CART} = 0.5$ because $\eta_2(X_2)$ is reflectional symmetric around $\eta_2(0.5)$. This would not target any subpopulations from $t_2$, as both the left node mean and right node mean are smaller than 0.75.
\begin{align*}
    &P(Y = 1| X_2 < 0.5) = \frac{\int_{0}^{0.5} \frac{9}{11} x dx}{0.5} = \frac{9}{44} < 0.75 \\
    &P(Y = 1| X_2 > 0.5) = \frac{\int_{0.5}^{1} \frac{9}{11} x dx}{0.5} = \frac{27}{44} < 0.75
\end{align*}
On the other hand, the best split for LPC is $s^* = \frac{11}{12}$. This splits $t_2$ into two groups, the left node is entirely below the threshold whereas the right node is entirely above, as illustrated by the orange and green segment in Figure \ref{Figure Monotonic example}, respectively. As a result, splitting at $s^*$ targets subpopulation $\{\frac{11}{12}<X_2<1|t_2\}$.

To summarize, splitting nodes $t_1$ and $t_2$ individually at $s^{CART}$ versus $s^*$ results in different target subpopulations: (i) $s^{CART}$: Target $\{X_1<\frac{1}{2}|t_1\}$; (ii) $s^*$: Target $\{X_1<\frac{5}{12}|t_1\}$ and $\{\frac{11}{12}<X_2<1|t_2\}$.
Assuming that both $t_1$ and $t_2$ contain the same amount of population, the two sets of policies target the same proportion of the population, but $\eta(x,t_1) < 0.75$ for $x \in \{\frac{5}{12}<X_1<\frac{1}{2}\}$, which is targeted by $s^{CART}$, whereas $\eta(x,t_2) > 0.75$ for $x \in \{\frac{11}{12}<X_2<1\}$, which is targeted by $s^*$. Therefore, policies based on LPC, i.e., $s^*$ policy, target a \textbf{more vulnerable} subpopulation than policies based on observed $Y$ classification, i.e., $s^{CART}$ policy or CART/KD-CART policy. This idea appears in our diabetes empirical study, which we detail in Section \ref{sec empirical studies}.

\textbf{Proof of Remark \ref{remark generalizing WRAcc}}
\begin{align*}
    \mathcal{G}^{\text{WRacc}} =& \text{WRacc of left node} + \text{WRacc of right node} \\
=& F(s) \left| \frac{\int_0^s \eta(x)  dF(x)}{F(s)} - \int_0^1 \eta(x)  dF(x) \right| + (1 - F(s)) \left| \frac{\int_s^1 \eta(x)  dF(x)}{1 - F(s)} - \int_0^1 \eta(x)  dF(x) \right| \\
=& F(s) \left| \mu_L - \int_0^1 \eta(x)  dF(x) \right| + (1 - F(s)) \left| \mu_R - \int_0^1 \eta(x)  dF(x) \right| \\
=& s \left| \mu_L - \int_0^1 \eta(x)  dx \right| + (1 - s) \left| \mu_R - \int_0^1 \eta(x)  dx \right| \quad \text{Under Assumption \ref{assumption unique intersection between eta and c}}
\end{align*}

\vspace{-2ex}
\section{Synthetic data simulation studies} \label{appendix Synthetic data simulation studies}
\vspace{-1ex}


\subsection{Data generation processes}\label{Data generation details}
\begin{enumerate}
    \item Generate features: \( X_i \sim U(0, 1), i\in \{1,2,3,4,5\} \). 
\item \begin{itemize}
    \item 
    \textbf{Ball:} \( f(X) = \sum_{i=1}^3 X_i^2  \).
    
    \item 
    \textbf{Friedman \#1:}  $f(X) = 10 \sin(\pi X_1  X_2) + 20(X_3 - 0.5)^2 + 10 X_4 + 5 X_5$
    
    \item 
    \textbf{Friedman \#2:}
    \begin{itemize}
    \item Make transformations: \(Z_1 = 100 X_1, Z_2 = 40\pi + 520\pi X_2, Z_4 = 10X_4 + 1\)
    \item Generate responses: \( f(X) = \sqrt{Z_1^2 + (Z_2 X_3  - \frac{1}{Z_2 Z_4}) ^2 }.\)
    \end{itemize}
    
    \item 
    \textbf{Friedman \#3:}
    \begin{itemize}
    \item Make transformations: \(Z_1 = 100 X_1, Z_2 = 40\pi + 520\pi X_2, Z_4 = 10X_4 + 1\)
    \item Generate responses: \( f(X) = \arctan\left((Z_2X_3 - \frac{1}{Z_2 Z_4})/Z_1\right).\)
    \end{itemize}
    
    \item 
    \textbf{Poly \#1:} $f(X) = 4X_1 + 3X_2^2 + 2X_3^3 + X_4^4$
    
    \item 
    \textbf{Poly \#2:} \(f(X) = X_1^4 + 2X_2^3 + 3X_3^2 + 4X_1\).
    
    \item 
    \textbf{Ring:} \( f(X) = |\sum_{i=1}^3 X_i^2 - 1|\).
    
    \item 
    \textbf{Collinear:}    
    \begin{itemize}
    \item Create correlated features: \( X_{i+3} = X_i + 0.1 \cdot \mathcal{N}(0, 1) \), where \( i\in\{1,2,3\}\).
    \item Generate responses: \( f(X) = \sum_{i=1}^6 X_i \).
    \end{itemize}
\end{itemize}
\item Map it to probabilities: $\eta = \text{Sigmoid}\left(f(X)-E(f(X))\right)$. 
\item Generate labels: \( y \sim \text{Bernoulli}(\eta) \).
\end{enumerate}
\vspace{-1.5ex}
\subsection{Full algorithms}\label{algor}
\vspace{-0.5ex}

We provide pseudo codes for CART, PFS, MDFS, wEFS, RF-CART, and RF-MDFS. In Algorithm~\ref{grow_tree}, setting `method' to CART, PFS, MDFS, or wEFS invokes the corresponding procedure.  Likewise, in Algorithm~\ref{grow_tree_wp}, choosing method = CART or MDFS produces RF-CART and RF-MDFS, respectively. The rest of the algorithms are all helper functions used in Algorithms~\ref{grow_tree} and \ref{grow_tree_wp}.
\newpage
\begin{breakablealgorithm}
\caption{Grow\_Tree}
\text{Fit the tree-based model to $(X,\bm{y})$ when $\eta$ or $\hat{\eta}$ is not provided.}
\begin{algorithmic}
    \State{\bfseries Input:}{$X, \bm{y}, c,$ current\_depth, method, depth, min\_samples}
    \If{current\_depth = depth \textbf{or} $n(\text{unique}(\bm{y}))=1$ \textbf{or} $n(\bm{y})<$min\_samples}
        \State{\bfseries Output:}{mean($\bm{y}$)}
    \EndIf
    \State best\_feature, best\_split $\gets$ FindBestSplit($X,\bm{y}$)
    \State mask $\gets (\bm{x}_{\text{best feature}}< \text{best\_split})$
    \State left\_X, right\_X $\gets X[\text{mask}], X[!\text{mask}]$
    \State left\_y, right\_y $\gets \bm{y}[\text{mask}], \bm{y}[!\text{mask}]$
    
    \If{current\_depth = depth - 1 \textbf{or} $\min\{n(\text{left\_y}), n(\text{right\_y)}|\}<$min\_samples \textbf{or} method!=`CART'}
        \If{method = `MDFS'}
            \State new\_best\_split $\gets$ Find\_PFS\_BestSplit($\bm{x}_{\text{best\_feature}}, \bm{y}, c, 1$)
        \ElsIf{method = `PFS'}
            \State new\_best\_split $\gets$ Find\_PFS\_BestSplit($\bm{x}_{\text{best\_feature}}, \bm{y}, c, 0.1$)
        \ElsIf{method = `wEFS'}
            \State new\_best\_split $\gets$ Find\_wEFS\_BestSplit($\bm{x}_{\text{best\_feature}}, \bm{y}, c$)
        \EndIf
        \State mask $\gets (\bm{x}_{\text{best\_feature}}< \text{new\_best\_split})$
        \State left\_tree, right\_tree $\gets$ mean($\bm{y}[\text{mask}]$), mean($\bm{y}[\text{!mask}]$)
        \State{\bfseries Output:}{best\_feature, new\_best\_split, left\_tree, right\_tree}
    \Else
        \State left\_tree = Grow\_Tree(left\_X, left\_y, current\_depth + 1, method, depth, min\_samples)
        \State right\_tree = Grow\_Tree(right\_X, right\_y, current\_depth + 1, method, depth, min\_samples)
        \State{\bfseries Output:}{best\_feature, best\_split, left\_tree, right\_tree}
    \EndIf
\end{algorithmic}\label{grow_tree}
\end{breakablealgorithm}

\begin{breakablealgorithm}
\caption{Grow\_Tree\_wP}
\text{Fit the tree-based model to $(X,\bm{y})$ and $\hat{\eta}$ provided by Knowledge-Distillation.}
\begin{algorithmic}
    \State{\bfseries Input:}{$X, \bm{y}, \bm{p}, c,$ current\_depth, method, depth, min\_samples}
    \If{current\_depth = depth \textbf{or} $(\min(\bm{p}) > c\ \textbf{or}\ \max(\bm{p}) < c)$ \textbf{or} $n(\bm{y})<$min\_samples}
        \State{\bfseries Output:}{mean($\bm{p}$)}
    \EndIf
    \State best\_feature, best\_split $\gets$ FindBestSplit($X,\bm{p}$)
    \State mask $\gets (\bm{x}_{\text{best feature}}< \text{best\_split})$
    \State left\_X, right\_X $\gets X[\text{mask}], X[!\text{mask}]$
    \State left\_y, right\_y $\gets \bm{y}[\text{mask}], \bm{y}[!\text{mask}]$
    \State left\_p, right\_p $\gets \bm{p}[\text{mask}], \bm{p}[!\text{mask}]$
    
    \If{current\_depth = depth - 1 \textbf{or} $\min\{n(\text{left\_p}), n(\text{right\_p)}|\}<$min\_samples \textbf{or} method!=`CART'}
        \If{method = `MDFS'}
            \State new\_best\_split $\gets$ Find\_PFS\_BestSplit($\bm{x}_{\text{best\_feature}}, \bm{y}, c, 1$)
        \ElsIf{method = `PFS'}
            \State new\_best\_split $\gets$ Find\_PFS\_BestSplit($\bm{x}_{\text{best\_feature}}, \bm{y}, c, 0.1$)
        \ElsIf{method = `wEFS'}
            \State new\_best\_split $\gets$ Find\_wEFS\_BestSplit($\bm{x}_{\text{best\_feature}}, \bm{p}, c$)
        \EndIf
        \State mask $\gets (\bm{x}_{\text{best\_feature}}< \text{new\_best\_split})$
        \State left\_tree, right\_tree $\gets$ mean($\bm{p}[\text{mask}]$), mean($\bm{p}[\text{!mask}]$)
        \State{\bfseries Output:}{best\_feature, new\_best\_split, left\_tree, right\_tree}
    \Else
        \State left\_tree = Grow\_Tree\_P(left\_X, left\_y, left\_p, current\_depth + 1, method, depth, min\_samples)
        \State right\_tree = Grow\_Tree\_P(right\_X, right\_y, right\_p, current\_depth + 1, method, depth, min\_samples)
        \State{\bfseries Output:}{best\_feature, best\_split, left\_tree, right\_tree}
    \EndIf
\end{algorithmic}\label{grow_tree_wp}
\end{breakablealgorithm}

\begin{breakablealgorithm}
\caption{Calculate\_impurity}\text{Calculate sample impurity $\hat{\mathcal{G}}^{CART}$.}\label{Calculate_impurity}
\begin{algorithmic}
    \State{\bfseries Input:} {$\bm{x}, \bm{y}, x$} 
    \State $\bm{y}_l \gets \bm{y}[\bm{x}\le x], \bm{y}_r \gets \bm{y}[\bm{x}>x]$ 
    \State $\mathcal{G} \gets \text{Var}(\bm{y}_l)\frac{|\bm{y}_l|}{n} + \text{Var}(\bm{y}_r)\frac{|\bm{y}_r|}{n}$
    \State{\bfseries Output:} {$\mathcal{G}$}
\end{algorithmic}
\end{breakablealgorithm}

\begin{breakablealgorithm}
\caption{Calculate\_distance}\text{Calculate the second term in \eqref{G_pfs} and $\hat{\mathcal{G}}^*$.}\label{Calculate_distance}
\begin{algorithmic}
    \State{\bfseries Input:} {$\bm{x}, \bm{y}, x, c$} 
    \State $\bm{y}_l \gets \bm{y}[\bm{x}\le x], \bm{y}_r \gets \bm{y}[\bm{x}>x]$ 
    \State $\mathcal{G} \gets (1-|c-\bar{y_l}|)\frac{|\bm{y}_l|}{n} + (1-|c-\bar{y_r}|)\frac{|\bm{y}_r|}{n}$
    \State{\bfseries Output:} {$\mathcal{G}$}
\end{algorithmic}
\end{breakablealgorithm}

\begin{algorithm} 
\caption{Calculate\_weighted\_risk}\text{Calculate sample loss function of wEFS.}\label{Calculate_risk} 
\begin{algorithmic}
    \State{\bfseries Input:}{$\bm{x}, \bm{y}, x, c$} 
    \State $\bm{y}_l \gets \bm{y}[\bm{x}\le x], \bm{y}_r \gets \bm{y}[\bm{x}>x]$ 
    \State $r_l \gets |\bm{y}_l[\bm{y}_l > c]|, r_r \gets |\bm{y}_r[\bm{y}_r > c]|$ 
    \State $\mathcal{R} \gets (\mathbbm{1}\{\bar{y}_l>c\}(|\bm{y}_l| - r_l) + \mathbbm{1}\{\bar{y}_r>c\}(|\bm{y}_r| - r_r))c + (\mathbbm{1}\{\bar{y}_l\le c\} r_l + \mathbbm{1}\{\bar{y}_r\le c\}r_r)(1-c)$ 
    \State{\bfseries Output:} {$\mathcal{R}$}
\end{algorithmic}
\end{algorithm}

\begin{algorithm}
\caption{FindBestSplit}\text{Take in data from the parent node and outputs the best feature, split point in terms of $\mathcal{G}^{CART}$.}\label{best_split}
\begin{algorithmic}
    \State{\bfseries Input:}{$X, \bm{y}$} 
    \State best\_impurity, best\_split, best\_feature $\gets$ 0, None, None
    \For{each feature $\bm{x}_i$ in $X$, $i$ from $1$ to $p$}
        \State sort $X$ in terms of $\bm{x}_i$
        \For{$x$ in ordered samples in $\bm{x}_i$}
            \State impurity $\gets$ Calculate\_impurity($\bm{x}_i, \bm{y}, x$) 
            \If{impurity $<$ best\_impurity}
                \State best\_impurity $\gets$ impurity
                \State best\_split $\gets$ $x$
                \State best\_feature $\gets$ $i$
            \EndIf
        \EndFor
    \EndFor
    \State{\bfseries Output:}{best\_split, best\_feature}
\end{algorithmic}
\end{algorithm}

\begin{algorithm}
\caption{Find\_PFS\_BestSplit}
\text{Take in data from the parent node and outputs the best feature, split point in terms of $\mathcal{G}^{PFS}(\lambda)$.}
\label{best_PFS_split}
\begin{algorithmic}
    \State{\bfseries Input:}{$\bm{x}, \bm{y}, c, \lambda$}
    \State best\_G\_PFS, best\_split $\gets$ 0, None
    \State sort $\bm{x}$
    \For{$x$ in sorted $\bm{x}$}
        \State G\_PFS $\gets$ $(1-\lambda)$ Calculate\_impurity($\bm{x}, \bm{y}, x$) + $\lambda$ Calculate\_distance($\bm{x}, \bm{y}, x, c$)
        \If{G\_PFS $<$ best\_G\_PFS}
            \State best\_G\_PFS $\gets$ G\_PFS
            \State best\_split $\gets$ $x$
        \EndIf
    \EndFor
    \State{\bfseries Output:}{best\_split}
\end{algorithmic}
\end{algorithm}

\begin{algorithm}
\caption{Find\_wEFS\_BestSplit}
\text{Take in data from the parent node and outputs the best feature, split point in terms of the weighted empirical risk.}\label{best_ERMFS_split}
\begin{algorithmic}
    \State{\bfseries Input:}{$\bm{x}, \bm{y}, c$}
    \State best\_risk, best\_split $\gets$ 0, None
    \State sort $\bm{x}$
    \For{$x$ in sorted $\bm{x}$}
        \State risk $\gets$ Calculate\_weighted\_risk($\bm{x}, \bm{y}, x, c$)
        \If{risk $<$ best\_risk}
            \State best\_risk $\gets$ risk
            \State best\_split $\gets$ $x$
        \EndIf
    \EndFor
    \State{\bfseries Output:}{best\_split}
\end{algorithmic}
\end{algorithm}

\clearpage

\textbf{How to implement KD-CART/KD-MDFS}: We first use \textit{RandomForestClassifier} from \textit{sklearn} \citep{scikit-learn} to train a random forest with the entire sample. All user-defined parameters of the random forest follow the default settings. Then, we grow an RF-CART and an RF-MDFS tree with the predicted probability output by the random forest as the response variable.

\subsection{Choice of hyperparameter \texorpdfstring{$\lambda$}{Lg}}\label{Appendix: honest approach}
Our simulations regarding PFS choose a small $\lambda$ value of 0.1, and overall, this choice leads to lower misclassification error than CART. To identify an appropriate $\lambda$ in practice, we had a standard procedure: we select the $\lambda$ value using an approach that combines the idea of cross-validation and the honest approach \citep{athey2016recursive}:

\begin{enumerate}
  \item Divide the data into $K = 5$ folds. For a fixed candidate $\lambda$ value, by leaving out fold $k \in \{1, 2, 3, 4, 5\}$, we can construct 5 trees, $T^{-k}$. Denote $T^{-k}$'s estimate for $X$ as $T^{-k}(X)$.
  
  \item Obtain an unbiased estimate of $\eta$ for each node by feeding the left-out fold to $T^{-k}$. Denote this estimate for $X$ as $\hat{\eta}^k(X)$.
  
  \item Calculate misclassification score (MS) (i.e., \textit{honest} approach), 
  \[
  \text{MS} = \sum_{k=1}^{5} \sum_{i \in \text{fold }k} \mathbf{1}\left\{(T^{-k}(X_i) - c)(\hat{\eta}^k(X_i) - c) < 0\right\}.
  \]
  
  \item Choose $\lambda$ value that minimizes MS.
\end{enumerate}

We implemented this approach in simulation for PFS, the results are comparable to setting $\lambda = 0.1$.

 
\subsection{Additional simulation results}\label{Appendix: simulation result}
Refer to Figure \ref{Boxplot simulation results in appendix F1} and Table \ref{Table simulation results in appendix}.

\begin{figure*}[htb]
    \centering
    \includegraphics[width=\linewidth]{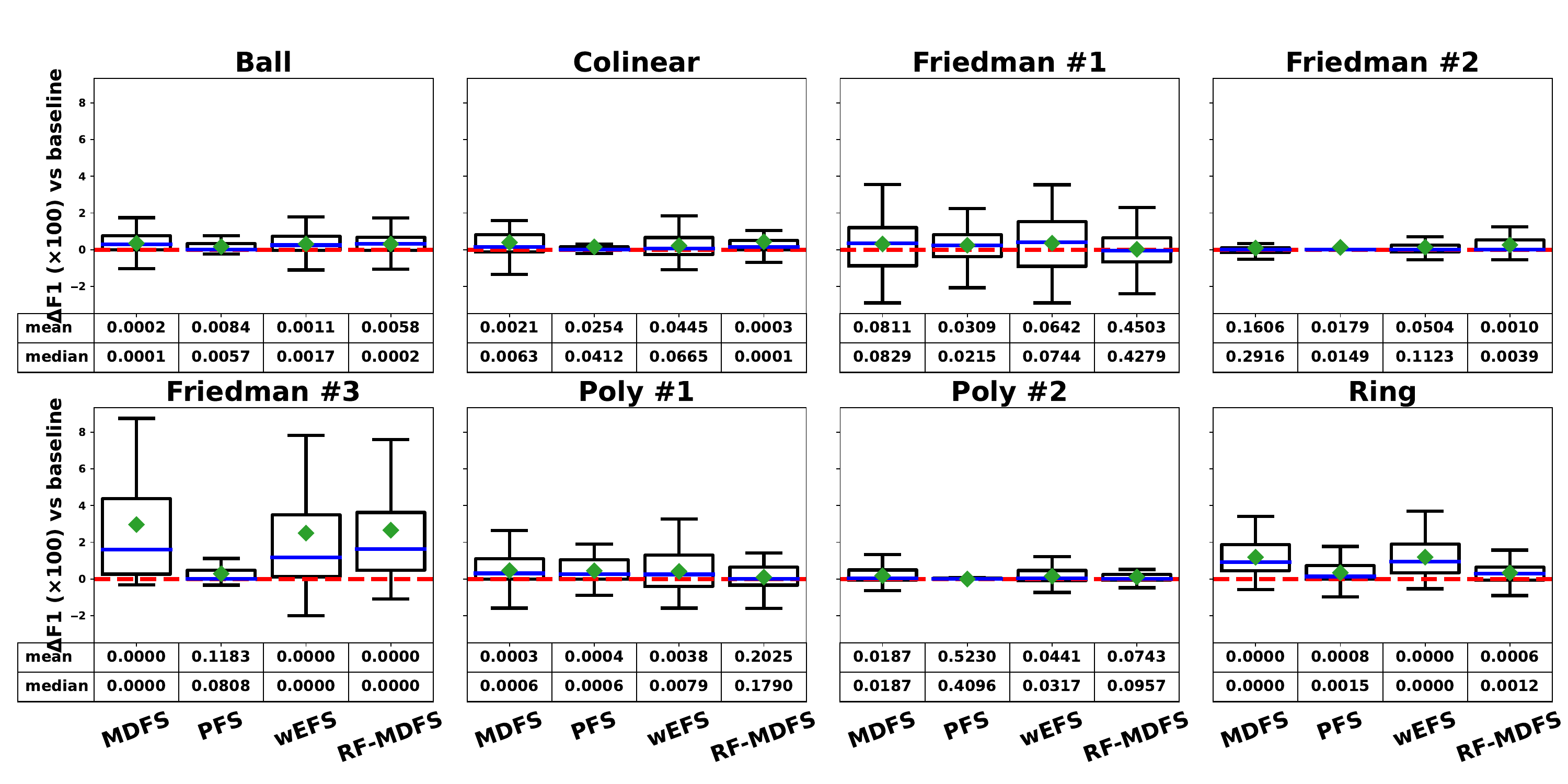}
    \caption{
    Boxplots of F1 differences relative to baseline models. For each panel, the first three boxplots compare CART with PFS, MDFS, and wEFS, and the last boxplot compares RF-CART with RF-MDFS. 
    Positive values in the boxplots denote \textit{improvement} in F1. Embedded tables list one-sided paired t-test and Wilcoxon signed-rank p-values for mean and median F1 differences, respectively. See boxplots for other settings in the supplemental files.}
    
    
    \label{Boxplot simulation results in appendix F1}
\end{figure*}

\begin{table}[htb]\tiny
\centering
\begin{tabular}{@{}lc|rrrrr|rr@{}}
\toprule
 \textbf{DGP} & $c$ & \textbf{CART} & \textbf{PFS} & \textbf{MDFS} & \textbf{WRACC} & \textbf{wEFS} & \textbf{RF-CART} & \textbf{RF-MDFS} \\
 \midrule
 Ball                 & 0.8 & 9.7 (1.8) & 9.2 (1.7) & 9.2 (1.5) & 8.9 (1.4) & 8.8 (1.3) & 8.8 (1.8) & \textbf{8.5} (1.2) \\
                      & 0.7 & 13.8 (2.6) & 13.7 (2.3) & 13.3 (2.2) & 13.8 (2.7) & 14.0 (2.5) & 12.8 (2.1) & \textbf{12.6} (2.0) \\
                      & 0.6 & 14.4 (2.2) & 14.0 (1.8) & 13.8 (1.8) & 14.2 (2.2) & 13.9 (1.7) & 13.6 (1.4) & \textbf{13.0} (1.4) \\
                      & 0.5 & 11.3 (1.8) & 11.0 (1.9) & 10.7 (1.7) & 10.7 (1.6) & 10.7 (1.6) & 10.5 (1.7) & \textbf{10.2} (1.6) \\
 \midrule
 Friedman \#1          & 0.8 & 6.7 (1.3) & 6.5 (1.4) & 6.4 (1.4) & 6.2 (1.2) & 6.0 (1.1) & \textbf{6.0} (1.0) & 6.3 (1.0) \\
                      & 0.7 & 12.5 (1.6) & 12.2 (1.7) & 12.2 (1.3) & 12.6 (1.4) & 12.6 (1.7) & 11.5 (1.2) & \textbf{11.5} (1.1) \\
                      & 0.6 & 18.5 (1.7) & 18.1 (1.6) & 18.0 (1.4) & 18.4 (1.6) & 18.5 (1.6) & 17.6 (1.4) & \textbf{17.2} (1.6) \\
                      & 0.5 & 20.4 (2.0) & 20.3 (1.5) & 20.2 (1.4) & 20.3 (1.4) & 20.1 (1.4) & \textbf{19.3} (1.3) & 19.8 (1.3) \\
 \midrule
 Friedman \#2          & 0.8 & 8.6 (2.0) & 7.8 (1.4) & \textbf{7.2} (1.3) & 8.2 (1.9) & 8.2 (1.9) & 8.4 (1.7) & 7.3 (1.4) \\
                      & 0.7 & 8.0 (1.7) & 7.6 (1.6) & 7.4 (1.4) & 7.8 (1.6) & 7.9 (1.5) & 7.8 (2.0) & \textbf{7.1} (1.6) \\
                      & 0.6 & 6.9 (1.7) & 6.6 (1.4) & 6.4 (1.5) & 6.8 (1.6) & 6.4 (1.4) & 6.5 (1.8) & \textbf{6.0} (1.5) \\
                      & 0.5 & 5.2 (1.5) & 5.0 (1.4) & 5.0 (1.2) & 5.0 (1.3) & 5.0 (1.3) & 5.0 (1.5) & \textbf{4.6} (1.0) \\
 \midrule
 Friedman \#3          & 0.8 & 1.9 (0.4) & 2.0 (0.4) & 1.8 (0.4) & 2.2 (0.5) & 2.1 (0.5) & 1.8 (0.4) & \textbf{1.8} (0.4) \\
                      & 0.7 & 2.6 (0.6) & 2.5 (0.5) & \textbf{2.3} (0.4) & 2.9 (0.7) & 2.5 (0.5) & 2.4 (0.5) & 2.4 (0.4) \\
                      & 0.6 & 3.2 (0.7) & 3.2 (0.6) & 2.9 (0.6) & 3.5 (0.6) & \textbf{2.9} (0.5) & 3.1 (0.5) & 3.0 (0.5) \\
                      & 0.5 & 4.0 (0.7) & 3.9 (0.7) & 3.6 (0.5) & 3.6 (0.5) & 3.6 (0.5) & 3.8 (0.7) & \textbf{3.6} (0.5) \\
 \midrule
 Poly \#1              & 0.8 & 7.3 (1.4) & 7.2 (1.3) & 7.3 (1.3) & 7.4 (1.3) & 7.4 (1.2) & 7.1 (1.5) & \textbf{6.8} (1.5) \\
                      & 0.7 & 10.6 (2.0) & 10.4 (2.0) & 10.1 (1.7) & 10.4 (1.9) & 10.6 (2.1) & 9.8 (1.9) & \textbf{9.3} (1.6) \\
                      & 0.6 & 12.5 (1.8) & 12.4 (2.8) & 12.2 (2.6) & 12.9 (2.5) & 12.9 (2.2) & \textbf{11.8} (2.0) & 11.8 (2.2) \\
                      & 0.5 & 13.0 (1.9) & 12.3 (2.0) & 12.3 (1.9) & 12.3 (1.7) & 12.3 (1.7) & 12.2 (2.1) & \textbf{12.0} (1.7) \\
 \midrule
 Poly \#2              & 0.8 & 10.0 (2.1) & 9.7 (1.9) & 9.7 (1.9) & 9.9 (2.1) & 9.9 (2.1) & 9.3 (2.1) & \textbf{9.0} (2.0) \\
                      & 0.7 & 9.7 (2.1) & 9.5 (2.0) & 9.4 (1.9) & 9.7 (2.0) & 9.9 (1.9) & 9.5 (2.0) & \textbf{9.3} (1.8) \\
                      & 0.6 & 8.2 (1.6) & 8.1 (1.6) & 7.9 (1.6) & 8.1 (1.9) & 8.0 (1.8) & 7.8 (1.4) & \textbf{7.6} (1.3) \\
                      & 0.5 & 6.9 (1.5) & 6.9 (1.5) & 6.7 (1.3) & 6.6 (1.4) & 6.6 (1.4) & 6.3 (1.6) & \textbf{6.2} (1.4) \\
 \midrule
 Ring                 & 0.8 & 17.1 (2.0) & 17.0 (2.0) & 16.6 (2.0) & 17.4 (1.8) & 17.5 (1.8) & 16.1 (1.7) & \textbf{15.9} (1.6) \\
                      & 0.7 & 14.5 (1.7) & 14.4 (1.4) & 14.3 (1.5) & 15.1 (1.4) & 14.9 (1.4) & 13.7 (1.3) & \textbf{13.3} (1.2) \\
                      & 0.6 & 12.7 (1.4) & 12.5 (1.4) & 11.9 (1.3) & 12.6 (1.4) & 12.2 (1.3) & 11.8 (1.2) & \textbf{11.4} (1.1) \\
                      & 0.5 & 10.7 (1.6) & 10.4 (1.5) & 9.8 (0.9) & 9.8 (0.9) & 9.8 (0.9) & 10.0 (1.2) & \textbf{9.5} (1.1) \\
 \midrule
 Colinear             & 0.8 & 6.5 (1.4) & 6.3 (1.3) & 5.8 (1.2) & 6.7 (1.3) & 6.6 (1.3) & 6.2 (1.2) & \textbf{5.8} (1.1) \\
                      & 0.7 & 7.8 (1.4) & 7.7 (1.4) & 7.3 (1.1) & 8.1 (1.4) & 8.2 (1.4) & 7.3 (1.6) & \textbf{7.0} (1.2) \\
                      & 0.6 & 8.3 (1.4) & 8.2 (1.3) & 8.0 (1.2) & 9.1 (1.6) & 8.6 (1.6) & 8.2 (1.5) & \textbf{7.7} (1.2) \\
                      & 0.5 & 8.5 (1.6) & 8.4 (1.5) & 8.1 (1.5) & 8.2 (1.4) & 8.2 (1.4) & 8.2 (1.4) & \textbf{7.7} (1.2) \\
\bottomrule
\end{tabular}
\caption{MR comparison of misclassification rate among all six methods. CART, PFS, MDFS, and wEFS use raw data input; CART and MDFS use random forest as a KD tool. Each entry presents the average and standard error (in parentheses) of the misclassification rate over 50 replicates. Note that the standard error is mostly from the data generation randomness, not the variability of the methods. To check for statistical significance, one should check out the t-test and rank sum test p-values in Figure \ref{Boxplot simulation results in main text}. The complete simulation results for the simulation setups presented in Table \ref{Table simulation results in appendix} and Figure \ref{Boxplot simulation results in main text} are provided in the supplemental files. }
\label{Table simulation results in appendix}
\end{table}

\begin{table}[htb]\tiny
\centering
\begin{tabular}{@{}lc|rrrrr|rr@{}}
\toprule
 \textbf{DGP} & $c$ & \textbf{CART} & \textbf{PFS} & \textbf{MDFS} & \textbf{WRACC} & \textbf{wEFS} & \textbf{RF-CART} & \textbf{RF-MDFS} \\
 \midrule
 Ball                 & 0.2 & 78.9 (4.0) & 79.5 (3.4) & 80.0 (3.4) & 80.1 (3.3) & 80.3 (3.3) & 80.4 (3.4) & \textbf{80.8} (3.6) \\
                      & 0.3 & 83.7 (2.6) & 83.8 (2.6) & 84.1 (2.6) & 83.6 (2.9) & 83.4 (2.8) & 84.8 (2.6) & \textbf{85.1} (2.4) \\
                      & 0.4 & 88.4 (1.9) & 88.7 (1.7) & 88.9 (1.1) & 88.7 (2.0) & 89.1 (1.4) & 89.0 (1.3) & \textbf{89.6} (1.2) \\
                      & 0.5 & 92.8 (1.3) & 93.0 (1.3) & 93.2 (1.2) & 93.1 (1.1) & 93.1 (1.1) & 93.3 (1.2) &\textbf{93.5} (1.1) \\
 \midrule
 Friedman \#1          & 0.2 & 45.6 (5.7) & 46.2 (5.7) & 46.7 (5.4) & 43.1 (6.1) & 45.9 (6.0) & \textbf{49.4} (5.4) & 48.4 (5.4) \\
                      & 0.3 & 61.8 (3.7) & 62.5 (4.8) & 63.0 (2.8) & 61.5 (3.0) & 62.3 (3.0) & 63.3 (3.8) & \textbf{63.8} (3.1) \\
                      & 0.4 & 71.4 (2.5) & 71.7 (2.7) & 71.7 (2.4) & 71.7 (2.4) & 72.5 (2.2) & 72.6 (2.3) & \textbf{73.0} (2.4) \\
                      & 0.5 & 80.6 (2.1) & 80.6 (1.7) & 80.8 (1.7) & 80.6 (1.7) & 80.8 (1.7) & \textbf{81.6} (1.7) & 81.2 (1.6) \\
 \midrule
 Friedman \#2          & 0.2 & 91.0 (2.3) & 91.8 (1.5) & 92.5 (1.4) & 91.5 (1.9) & 91.5 (1.9) & 91.2 (1.7) & \textbf{92.4} (1.4) \\
                      & 0.3 & 93.7 (1.4) & 94.0 (1.3) & 94.2 (1.1) & 93.8 (1.3) & 93.7 (1.3) & 93.8 (1.7) & \textbf{94.3} (1.3) \\
                      & 0.4 & 95.3 (1.2) & 95.5 (1.0) & 95.7 (1.0) & 95.5 (1.1) & 95.8 (0.9) & 95.6 (1.3) & \textbf{96.0} (1.1) \\
                      & 0.5 & 96.9 (0.9) & 97.0 (0.9) & 96.9 (0.7) & 97.0 (0.8) & 97.0 (0.8) & 96.9 (0.9) & \textbf{97.2} (0.6) \\
 \midrule
 Friedman \#3          & 0.2 & 63.1 (7.7) & 61.8 (13.4) & \textbf{64.7} (9.5) & 50.2 (24.4) & 59.2 (15.6) & 63.8 (10.6) & 62.1 (13.6) \\
                      & 0.3 & 69.3 (5.0) & 70.2 (5.2) & \textbf{71.6} (5.5) & 66.3 (10.8) & 71.4 (5.2) & 70.1 (5.5) & 70.7 (6.3) \\
                      & 0.4 & 73.6 (4.4) & 74.2 (4.3) & \textbf{75.9} (4.1) & 73.6 (3.7) & 75.7 (3.8) & 74.3 (3.3) & 75.1 (3.8) \\
                      & 0.5 & 76.2 (3.3) & 76.5 (3.0) & 77.7 (3.0) & 77.2 (3.2) & 77.2 (3.2) & 76.9 (3.2) & \textbf{77.7} (2.6) \\
 \midrule
 Poly \#1              & 0.2 & 74.7 (3.9) & 75.9 (4.5) & 76.3 (3.8) & 75.2 (6.6) & 75.5 (4.5) & 75.9 (5.2) & \textbf{76.6} (5.4) \\
                      & 0.3 & 81.9 (3.3) & 82.4 (3.4) & 82.7 (2.9) & 82.2 (3.2) & 82.0 (3.3) & 83.0 (3.5) & \textbf{83.9} (2.7) \\
                      & 0.4 & 86.7 (1.9) & 86.8 (2.9) & 87.1 (2.5) & 86.8 (2.5) & 86.9 (2.0) & 87.3 (2.3) & \textbf{87.5} (2.2) \\
                      & 0.5 & 90.0 (1.5) & 90.5 (1.6) & 90.5 (1.5) & 90.5 (1.4) & 90.5 (1.4) & 90.6 (1.7) & \textbf{90.7} (1.4) \\
 \midrule
 Poly \#2              & 0.2 & 90.3 (2.1) & 90.4 (1.9) & 90.4 (1.9) & 90.2 (2.2) & 90.2 (2.1) & 90.8 (2.0) & \textbf{91.1} (1.8) \\
                      & 0.3 & 92.7 (1.6) & 92.9 (1.5) & 92.9 (1.4) & 92.7 (1.4) & 92.6 (1.4) & 92.8 (1.5) & \textbf{93.0} (1.3) \\
                      & 0.4 & 94.9 (1.0) & 94.9 (1.0) & 95.0 (1.0) & 95.0 (1.2) & 95.0 (1.1) & 95.1 (0.9) & \textbf{95.2} (0.9) \\
                      & 0.5 & 96.1 (0.9) & 96.1 (0.9) & 96.2 (0.8) & 96.2 (0.8) & 96.2 (0.8) & 96.4 (0.9) & \textbf{96.5} (0.8) \\
 \midrule
 Ring                 & 0.2 & 70.2 (3.1) & 71.0 (3.4) & 71.4 (3.2) & 69.8 (2.9) & 70.2 (2.9) & 71.9 (2.9) & \textbf{72.4} (2.7) \\
                      & 0.3 & 82.5 (2.0) & 82.6 (1.5) & 82.8 (1.7) & 82.3 (1.5) & 82.6 (1.4) & 83.4 (1.6) & \textbf{83.8} (1.4) \\
                      & 0.4 & 88.1 (1.3) & 88.3 (1.3) & 89.0 (1.2) & 88.6 (1.2) & 88.8 (1.1) & 89.0 (1.2) & \textbf{89.4} (1.1) \\
                      & 0.5 & 91.7 (1.3) & 92.0 (1.2) & 92.5 (0.7) & 92.5 (0.7) & 92.5 (0.7) & 92.3 (0.9) & \textbf{92.7} (0.8) \\
 \midrule
 Colinear             & 0.2 & 81.3 (3.8) & 81.5 (3.6) & 82.9 (3.3) & 80.6 (3.4) & 80.8 (3.3) & 81.8 (3.5) & \textbf{83.3} (3.0) \\
                      & 0.3 & 85.9 (2.5) & 86.3 (2.5) & 86.8 (2.2) & 85.6 (2.6) & 85.5 (2.5) & 86.9 (2.8) & \textbf{87.3} (2.2) \\
                      & 0.4 & 89.2 (1.9) & 89.4 (1.8) & 89.7 (1.6) & 88.8 (1.9) & 89.3 (1.8) & 89.6 (1.7) & \textbf{90.2} (1.6) \\
                      & 0.5 & 91.5 (1.7) & 91.6 (1.7) & 91.9 (1.5) & 91.7 (1.5) & 91.7 (1.5) & 91.8 (1.5) & \textbf{92.2} (1.2) \\
\bottomrule
\end{tabular}
\caption{F1 comparison of misclassification rate among all six methods. CART, PFS, MDFS, and wEFS use raw data input; CART and MDFS use random forest as a KD tool. Each entry presents the average and standard error (in parentheses) of the misclassification rate over 50 replicates. Note that the standard error is mostly from the data generation randomness, not the variability of the methods. To check for statistical significance, one should check out the t-test and rank sum test p-values in Figure \ref{Boxplot simulation results in appendix F1}. The complete simulation results for the simulation setups presented in Table \ref{Table simulation results in appendix (F1)} and Figure \ref{Boxplot simulation results in appendix F1} are provided in the supplemental files. }
\label{Table simulation results in appendix (F1)}
\end{table}

\subsection{Computational complexity analysis} \label{Appendix: Computational Complexity}

Since our proposed algorithm only operates at the last split, the tree-building processes are identical in upper levels among 4 strategies. Next, we evaluate the computational complexity of the \textbf{final split}, i.e., split at the final level. 

Suppose the sample size in a node at the second to last level is $n$ and the number of features is $p$. In \textbf{CART}, the dominant step is to sort the samples based on each feature. With $p$ features to consider, the computational complexity is $O(pNlogN)$. \textbf{PFS}, \textbf{MDFS}, and \textbf{wEFS} all determine the feature to split on using \textbf{CART}. The split point decision requires another $O(N)$ of time which is still dominated by the sorting process.

\section{Additional empirical study details}\label{Appendix Empirical Studies}
We use publicly available data from our empirical studies. The Pima Indian Diabetes dataset can be downloaded at \url{https://www.kaggle.com/datasets/uciml/pima-indians-diabetes-database}. The Montesinho Park forest fire dataset can be downloaded at \url{https://archive.ics.uci.edu/dataset/162/forest+fires}. 

In this section, we use forest fire dataset to illustrate an application of our methods MDFS and RF-MDFS, and compare them to CART and RF-CART respectively.

The UCI Forest Fire dataset measures the characteristics of different forest fires in Montesinho Park with the response variable being the area burnt by a forest fire. We binarize the response by labeling those observations with a burnt area larger than 5 as 1, and 0 otherwise. 1 indicates a ``big'' forest fire. We search for conditions under which the probability of having a big forest fire is above 1/3. We use the following features in the dataset for this task: X and Y-axis spatial coordinates, month, FFMC, DMC, DC, ISI, temperature, RH (relative humidity), wind, and rain. All the acronyms except for RH are measures that positively correlate with the probability of a big forest fire. With a moderate sample size of 517, we set $m=3$. 

\begin{figure*}[ht] 
    \centering
    \begin{minipage}{0.5\textwidth}
        \centering
        \fontsize{8pt}{9.6pt}\selectfont
        \textbf{CART}
            \begin{Verbatim}[commandchars=\\\{\}]
if FFMC <= 85.85
    if temp <= 5.15
        if X <= 5.0
            \textcolor[HTML]{990000}{value: 1.000, samples: 8}
            \textcolor[HTML]{990000}{value: 0.600, samples: 5}
        if temp <= 18.55
            value: 0.172, samples: 29
            \textcolor[HTML]{990000}{value: 0.667, samples: 9}
    if temp <= 26.0
        if DC <= 673.2
            value: 0.208, samples: 216
            value: 0.310, samples: 200
        if RH <= 24.5
            {value: 0.125, samples: 8} 
            \textcolor[HTML]{990000}{value: 0.500, samples: 42}
        \end{Verbatim}
    \end{minipage}%
    \hfill
    \begin{minipage}{0.5\textwidth}
        \centering
        \fontsize{8pt}{9.6pt}\selectfont
        \textbf{MDFS}
        \begin{Verbatim}[commandchars=\\\{\}]
if FFMC <= 85.85
    if temp <= 5.15
        if X <= 3.5
            \textcolor[HTML]{990000}{value: 1.000, samples: 2}
            \textcolor[HTML]{990000}{value: 0.818, samples: 11}
        if temp <= 18.55
            value: 0.172, samples: 29
            \textcolor[HTML]{990000}{value: 0.667, samples: 9}
    if temp <= 26.0
        if DC <= 766.2
            value: 0.244 samples: 216
            \textcolor[HTML]{990000}{value: 0.372, samples: 43}
        if RH <= 24.5
            {value: 0.125, samples: 8} 
            \textcolor[HTML]{990000}{value: 0.500, samples: 42}
        \end{Verbatim}
    \end{minipage}
\\
\vspace{.5cm}
\begin{minipage}{0.5\textwidth}
        \centering
        \fontsize{8pt}{9.6pt}\selectfont
        \textbf{RF-CART}
            \begin{Verbatim}[commandchars=\\\{\}]
if FFMC <= 85.85
    if temp <= 5.15
        if wind <= 7.15
            \textcolor[HTML]{990000}{value: 0.638, samples: 6}
            \textcolor[HTML]{990000}{value: 0.964, samples: 7}
        if temp <= 18.55
            value: 0.229, samples: 29
            \textcolor[HTML]{990000}{value: 0.594, samples: 9}
    if temp <= 26.0
        if DC <= 673.2
            value: 0.220 samples: 216
            value: 0.308, samples: 200
        if FFMC <= 95.85
            \textcolor[HTML]{990000}{value: 0.381, samples: 43} 
            \textcolor[HTML]{990000}{value: 0.649, samples: 7}
        \end{Verbatim}
    \end{minipage}%
    \hfill
    \begin{minipage}{0.5\textwidth}
        \centering
        \fontsize{8pt}{9.6pt}\selectfont
        \textbf{RF-MDFS}
        \begin{Verbatim}[commandchars=\\\{\}]
if FFMC <= 85.85
    if temp <= 5.15
        
        \textcolor[HTML]{990000}{value: 0.814, samples: 13}
            
        if temp <= 18.55
            value: 0.229, samples: 29
            \textcolor[HTML]{990000}{value: 0.594, samples: 9}
    if temp <= 26.0
        if DC <= 766.2
            value: 0.250 samples: 373
            \textcolor[HTML]{990000}{value: 0.368, samples: 43}
        if FFMC <= 91.95
            \textcolor[HTML]{990000}{value: 0.578, samples: 4} 
            \textcolor[HTML]{990000}{value: 0.404, samples: 46}
        \end{Verbatim}
    \end{minipage}

    \caption{The targeting policies generated by CART, MDFS, RF-CART, RF-MDFS. The \textcolor[HTML]{990000}{red} groups are the targeted subpopulations predicted to have a higher than 1/3 probability of catching a big forest fire.}
    \label{figure diabetes CART vs MDFS}
\end{figure*}

\subsection{CART vs MDFS}
As shown by Figure \ref{figure forest fire CART vs MDFS}, the two sets of policies target many common groups except for the subgroup consisting of 43 observations defined by $FFMC > 85.85$, $temp \leq 26.0$, $DC > 766.2$. This group that MDFS additionally targets has a 37.2\% probability of catching a big forest fire, higher than the threshold. This finding aligns with Remark \ref{remark more cost}.

\subsection{RF-CART vs RF-MDFS}
As shown by Figure \ref{figure forest fire CART vs MDFS}, RF-MDFS also additionally targets the subgroup consisting of 43 observations defined by $FFMC > 85.85$, $temp \leq 26.0$, $DC > 766.2$. Hence, the comparison between RF-CART and RF-MDFS here also aligns with Remark \ref{remark more cost}.

Interestingly, both RF-CART and RF-MDFS (compared to CART and MDFS respectively) additionally target the subgroup $FFMC > 85.85$, $temp > 26.0$, $RH \leq 24.5$.

\subsection{Full CART/MDFS/RF-CART/MDFS for the diabetes example}

\begin{figure*}[ht] 
    \centering
    \begin{minipage}{0.5\textwidth}
        \centering
        \fontsize{8pt}{9.6pt}\selectfont
        \textbf{CART}
            \begin{Verbatim}[commandchars=\\\{\}]
if Glucose <= 127.5
    if Age <= 28.5
        if BMI <= 30.95
            value: 0.013, samples: 151
            value: 0.175, samples: 120
        if BMI <= 26.35
            value: 0.049, samples: 41
            value: 0.399, samples: 173
    if BMI <= 29.95
        if Glucose <= 145.5
            value: 0.146, samples: 41
            value: 0.514, samples: 35
        if Glucose <= 157.5
            \textcolor[HTML]{990000}{value: 0.609, samples: 115} 
            \textcolor[HTML]{990000}{value: 0.870, samples: 92}
        \end{Verbatim}
    \end{minipage}%
    \hfill
    \begin{minipage}{0.5\textwidth}
        \centering
        \fontsize{8pt}{9.6pt}\selectfont
        \textbf{MDFS}
        \begin{Verbatim}[commandchars=\\\{\}]
if Glucose <= 127.5
    if Age <= 28.5
        if BMI <= 22.25
            value: 0.000, samples: 35
            value: 0.097, samples: 236
        if BMI <= 28.25
            value: 0.222, samples: 63
            value: 0.377, samples: 151
    if BMI <= 29.95
        if Glucose <= 166.5
            value: 0.250, samples: 64
            \textcolor[HTML]{990000}{value: 0.667, samples: 12}
        if Glucose <= 129.5
            value: 0.579, samples: 19 
            \textcolor[HTML]{990000}{value: 0.739, samples: 188}
        \end{Verbatim}
    \end{minipage}
\\
\vspace{.5cm}
\begin{minipage}{0.5\textwidth}
        \centering
        \fontsize{8pt}{9.6pt}\selectfont
        \textbf{RF-CART}
            \begin{Verbatim}[commandchars=\\\{\}]
if Glucose <= 127.5
    if Age <= 28.5
        if BMI <= 30.95
            value: 0.028, samples: 151
            value: 0.186, samples: 120
        if BMI <= 26.95
            value: 0.104, samples: 45
            value: 0.396, samples: 169
    if BMI <= 29.95
        if Glucose <= 145.5
            value: 0.193, samples: 41
            value: 0.511, samples: 35
        if Glucose <= 157.5
           value: 0.595, samples: 115 
            \textcolor[HTML]{990000}{value: 0.830, samples: 92}
        \end{Verbatim}
    \end{minipage}%
    \hfill
    \begin{minipage}{0.5\textwidth}
        \centering
        \fontsize{8pt}{9.6pt}\selectfont
        \textbf{RF-MDFS}
        \begin{Verbatim}[commandchars=\\\{\}]
if Glucose <= 127.5
    if Age <= 28.5
        if BMI <= 22.25
            value: 0.010, samples: 35
            value: 0.111, samples: 236
        if BMI <= 28.25
            value: 0.225, samples: 63
            value: 0.380, samples: 151
    if BMI <= 29.95
        if Glucose <= 166.5
            value: 0.284, samples: 64
            \textcolor[HTML]{990000}{value: 0.636, samples: 12}
        if Glucose <= 129.5
           value: 0.573, samples: 19 
            \textcolor[HTML]{990000}{value: 0.713, samples: 188}
        \end{Verbatim}
    \end{minipage}

    \caption{The targeting policies generated by CART, MDFS, RF-CART, RF-MDFS. The \textcolor[HTML]{990000}{red} groups are the targeted subpopulations predicted to a higher than 60\% probability of being diabetic.}
    \label{figure forest fire CART vs MDFS full}
\end{figure*}



\end{document}